%% file: samplepaper.tex
\documentclass[runningheads]{llncs}
\usepackage[T1]{fontenc}
\usepackage{algorithm}
\usepackage[noend]{algpseudocode}
\usepackage{wrapfig}
\usepackage{hyperref}
\usepackage{graphicx}
\usepackage{lipsum}
\usepackage{tcolorbox}
\tcbuselibrary{breakable}
\usepackage{wrapfig}
\usepackage{amsmath}
\usepackage{amssymb}
\usepackage{subcaption}
\usepackage{tabularx,array}
\usetikzlibrary{plotmarks,decorations.shapes,shapes.geometric,shadows}
\usepackage{color}

\usepackage{import}
\usepackage{xifthen}
\usepackage{pdfpages}
\usepackage{transparent}

\usepackage{orcidlink}
\renewcommand{\orcidID}{\orcidlink}
\allowdisplaybreaks
\definecolor{grayfilling}{gray}{0.95} % Used for tikz graphics as a filling
\definecolor{grayshadow}{gray}{0.5} % Used for tikz graphics as a filling
\definecolor{mydarkblue}{HTML}{003a7d}
\definecolor{mypurple}{HTML}{c701ff}
\definecolor{mymidblue}{HTML}{008dff}
\colorlet{myblue}{mymidblue}
\definecolor{mygreen}{HTML}{4ecb8d}
\definecolor{myorange}{HTML}{ff9d3a}
\definecolor{myyellow}{HTML}{f9e858}
\definecolor{myred}{HTML}{d83034}

\colorlet{accent}{mypurple}

\makeatletter
\renewcommand\paragraph{\@startsection{paragraph}{4}{\z@}%
        {-8\p@ \@plus -4\p@ \@minus -4\p@}%
        {-0.5em \@plus -0.22em \@minus -0.1em}%
        {\normalfont\bfseries}}%
\makeatletter

\usepackage{mdframed}
\newmdenv[
  backgroundcolor=gray!10,
  shadow=false,
  hidealllines=true,
  leftline=true,
  linecolor=gray,
  linewidth=3pt,
  skipabove=4pt,
  skipbelow=4pt,
  innerleftmargin=10pt,
  innerrightmargin=10pt,
  innertopmargin=8pt,
  innerbottommargin=8pt
]{graybox}
\newmdenv[
  backgroundcolor=accent!10,
  shadow=false,
  hidealllines=true,
  leftline=true,
  linecolor=accent,
  linewidth=3pt,
  skipabove=4pt,
  skipbelow=4pt,
  innerleftmargin=10pt,
  innerrightmargin=10pt,
  innertopmargin=8pt,
  innerbottommargin=8pt
]{accentbox}
\newtheorem{thm_}{Theorem}
\renewenvironment{theorem}
{\begin{accentbox}\begin{thm_}}
{\end{thm_}\end{accentbox}}
\newtheorem{cor_}{Corollary}
\renewenvironment{corollary}
{\begin{accentbox}\begin{cor_}}
{\end{cor_}\end{accentbox}}
\newtheorem{lem_}{Lemma}
\renewenvironment{lemma}
{\begin{accentbox}\begin{lem_}}
{\end{lem_}\end{accentbox}}
\spnewtheorem*{runexample}{Running Example}{\bfseries}{\itshape}

\DeclareMathOperator{\best}{best}
\newcommand{\cdotx}{\,\cdot\,}
\newcommand{\deltat}{{\delta_T}} % Temporal disturbance
\newcommand{\deltati}[1]{{\delta_{T,#1}}} % Temporal disturbance
\newcommand{\deltax}{{\delta_X}} % Spatial disturbance
\newcommand{\Deltat}{{\Delta_T}} % Temporal robustness
\newcommand{\Deltax}{{\Delta_X}} % Spatial robustness
\newcommand{\F}{\mathbf{F}} % Eventually operator
\newcommand{\G}{\mathbf{G}} % Globally/always operator
\newcommand{\norm}[1]{\|#1\|}
\newcommand{\N}{\mathbb{N}} % Natural numbers
\newcommand{\R}{\mathbb{R}} % Reel numbers
\DeclareMathOperator{\sd}{{sd}} % Signed distance function
\newcommand{\U}{\mathbf{U}} % Until operator
\newcommand{\Z}{\mathbb{Z}} % Integers

\begin{document}
\title{Spatiotemporal Robustness of Temporal Logic Tasks using Multi-Objective Reasoning}
\titlerunning{Spatiotemporal Robustness using Multi-Objective Reasoning}

\newif\ifanonymous
\anonymousfalse

\ifanonymous
\author{Anonymous Author(s)}
\institute{Anonymous Institution(s)}
\else
\author{Oliver Schön\orcidID{0000-0002-0214-6455}\thanks{Corresponding author} \and
Lars Lindemann\orcidID{0000-0003-3430-6625}}
\authorrunning{Oliver Schön and Lars Lindemann}
\institute{Automatic Control Laboratory, ETH Zürich, 
% Zürich 8092, 
Switzerland \\
\email{\{oschoen,llindemann\}@ethz.ch}}
\fi
\maketitle              % typeset the header of the contribution
\begin{abstract}
The reliability of autonomous systems depends on their robustness, i.e., their ability to meet their objectives under uncertainty. In this paper, we study spatiotemporal robustness of temporal logic specifications evaluated over discrete-time signals.  Existing work has proposed robust semantics that capture not only Boolean satisfiability, but also the geometric distance from unsatisfiability, corresponding to admissible spatial perturbations of a given signal.
In contrast, we  propose spatiotemporal robustness (STR), which captures admissible spatial and temporal perturbations jointly. This notion is particularly informative for interacting systems, such as  multi-agent robotics, smart cities, and air traffic control. We define STR as a multi-objective reasoning problem, formalized via a partial order over spatial and temporal perturbations. This perspective has two key advantages: (1) STR can be interpreted as a Pareto-optimal set that characterizes all admissible spatiotemporal perturbations, and (2) STR can be computed using tools from multi-objective optimization. To navigate computational challenges, we propose robust semantics  for STR that are sound in the sense of suitably under-approximating  STR while being computationally tractable. Finally, we present monitoring algorithms for STR using these robust semantics. To the best of our knowledge, this is the first work to deal with robustness across multiple dimensions via  multi-objective reasoning. 

\end{abstract}
\section{Introduction}

Consider the following three safety-critical applications: (1) drone fleets for wildfire monitoring, (2) autonomous taxis navigating urban traffic, and (3) air traffic ground control for avoiding runway incursions. While these applications are fundamentally different in terms of their underlying systems and objectives, they share some key characteristics when it comes to reasoning about their robustness and reliability. First, these applications are time-critical in the sense that  systems have to meet time-sensitive objectives.  Second, failure of the systems to achieve their objectives can have detrimental effects on the system’s safety and performance. Third, the systems are subject to uncertainties that can lead to spatial and temporal deviations from their intended behavior. These concerns are not merely theoretical, but can result in real accidents when robustness margins are violated, e.g., as in the 2025 Potomac River mid-air collision~\cite{Funk}. Such safety issues ultimately call for robust  monitoring, verification, and design techniques. 

These observations motivate us to study  the \emph{spatiotemporal robustness} (STR) w.r.t. temporal logic specifications. To set the stage, assume that we are given a signal $x(t)$ for times $t\ge 0$ (later assumed to be discrete-time), e.g., describing the position and velocity of an aircraft. Suppose also that the signal $x(t)$ satisfies a requirement $\phi$, denoted by $x\models \phi$.  In time-critical systems, such requirements are expressed as spatiotemporal objectives that specify both the `when' (temporal requirement) and the `what' (spatial requirement). While our definition of STR applies to any such requirement, we adopt temporal logic --- in particular signal temporal logic \cite{maler2004monitoring} --- to formally express spatiotemporal requirements. Temporal logic is widely used in robotics \cite{kress2009temporal,kantaros2020stylus,leung2023backpropagation}, control \cite{raman2014model,lindemann2018control,belta2017formal}, and autonomous systems \cite{bartocci2018specification,annpureddy2011s} due to its generality and interpretability in natural language \cite{chen2023nl2tl,mendoza2024translating}.

\begin{figure}[t]
    \centering
    \begin{subfigure}[b]{.49\linewidth}
        \centering
    \def\svgwidth{.8\linewidth}
    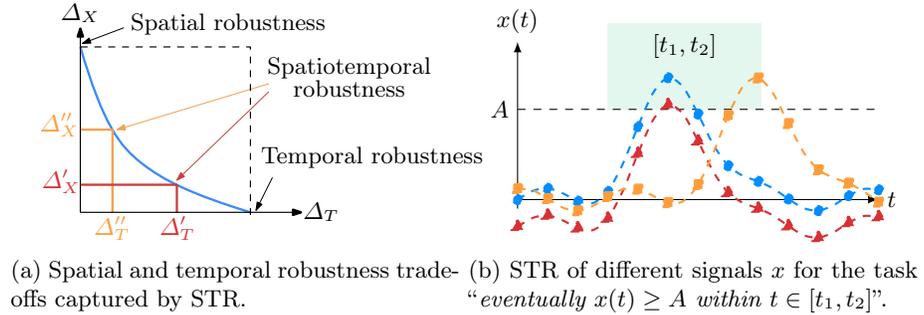

        \caption{Spatial and temporal robustness trade-offs captured by STR.}
        %alt={A graph with temporal robustness on the x-axis and spatial robustness on the y-axis. The Pareto front is shown as a blue curve, which is monotonically decreasing and symmetric around the line x=y.}
        \label{fig:Pareto}
    \end{subfigure}
    \begin{subfigure}[b]{.49\linewidth}
        \centering
        \begin{tikzpicture}[scale=0.6]
    \tikzset{every mark/.append style={scale=1.4, mark repeat=2, mark phase=1}}
        % Axes
        \draw[-latex] (0,0) -- (8,0) node[right] {$t$};
        \draw[-latex] (0,-0.2) -- (0,3.5) node[above] {$x(t)$};
    
        % Parameters
        \def\tone{2}
        \def\ttwo{5+.4}
        \def\A{2}
    
        % Time window
        \fill[mygreen!15] (\tone,\A) rectangle (\ttwo,3.9);
        \node at (3.5+.2,3.3+.1) {\small $[t_1,t_2]$};
    
        % Threshold
        \draw[dashed] (0,\A) -- (8,\A);
        \node[left] at (0,\A) {$A$};
    
        % Nominal signal
        \draw[thick, myblue, dashed, mark=*]
            plot[smooth, domain=0:8]
            (\x,{2.6*exp(-(\x-3.5)^2/1.2)+.25*sin(150*\x)});
        %\node[blue] at (6.2,2.6) {\small nominal};
    
        % Spatially perturbed signal (vertical shift)
        \draw[thick, myred, dashed, mark=triangle*, mark size=3]
            plot[smooth, domain=0:8]
            (\x,{2.6*exp(-(\x-3.5)^2/1.2)+.25*sin(150*\x)-.6});
        %\node[red] at (6.2,2.0) {\small spatial};
    
        % Temporally perturbed signal (horizontal shift)
        \draw[thick, myorange, dashed, mark=square*]
            plot[smooth, domain=0:8]
            (\x,{2.6*exp(-(\x-3.5-1.9)^2/1.2)+.25*sin(150*(\x-1.9))});
        %\node[orange] at (6.3,1.3) {\small temporal};
    \end{tikzpicture}
    \caption{STR of different signals $x$ for the task ``\emph{eventually $x(t)\geq A$ within $t\in[t_1,t_2]$}''.}
    %alt={Three signals are plotted over time, with a green shaded area indicating the time window [t1,t2] and a dashed line indicating the threshold A. The blue signal is the nominal signal, which exceeds the threshold within the time window. The red signal is a spatially perturbed version of the nominal signal, which is shifted downwards and barely exceeds the threshold within the time window. The orange signal is a temporally perturbed version of the nominal signal, which is shifted to the right and also barely meets the threshold within the time window.}
    \label{fig:motivating_example}
    \end{subfigure}
    \caption{Motivation for spatiotemporal robustness.}
    \label{fig:motivation}
\end{figure}

\emph{Spatial robustness} was proposed in \cite{fainekos2009robustness,donze2010robust} to measure  the maximal size $\Deltax\geq\norm{\deltax}_\infty$ of a  perturbation $\delta_x(t)$ that can be added to $x(t)$ without violating $\phi$, i.e., such that $x+\delta_x \models \phi$. Since then, spatial robustness metrics  have been transformative and are  widely used for analyzing safety (e.g., for collision avoidance \cite{ames2016control,ames2019control}) and designing robust systems \cite{deshmukh2017robust,gilpin2020smooth,li2017reinforcement,pant2017smooth}. 
Despite their widespread use, they have major shortcomings when it comes to capturing temporal --- and ultimately  spatiotemporal --- perturbations. Indeed, a challenge in many applications arises from variations in computation and actuation times and from complex timing uncertainties caused by human interaction, sensing delays and failures, and computation-intensive perception. \emph{Temporal robustness} metrics \cite{donze2010robust,lindemann2022temporal} capture the maximum size $\Deltat\geq |\deltat|$ of a time shift $\deltat$ that can be applied to $x(t)$ without violating $\phi$, i.e., such that $x'\models \phi$, where $x'(t)=x(t-\deltat)$ is a time shifted version of $x(t)$.\footnote{More generally, the time shift $\deltat$ could be a vector to model asynchronous time shifts within the signal $x$, as we choose to do later in this paper.} To the best of our knowledge, no existing work has thoroughly studied the setting of joint spatial and temporal perturbations.\footnote{The work in \cite{donze2010robust} also proposes to fix a spatial robustness level and to then compute the resulting temporal robustness.  While the authors briefly hint at a Pareto-optimal set interpretation for spatial and temporal perturbations, this is left  unexplored. } Our definition of STR will define Pareto-optimal sets of admissible disturbances in space and time that capture trade-offs between spatial and temporal robustness; see Fig.~\ref{fig:Pareto} for an illustration.

Figure~\ref{fig:motivating_example} shows three different  signals with varying degrees of robustness in {\color{myblue}blue}, {\color{myred}red}, and {\color{myorange}yellow}. Considering spatial robustness alone would favor the {\color{myorange}yellow} signal over the {\color{myred}red} one, whereas the opposite conclusion would be reached when considering temporal robustness alone. Existing work such as \cite{akazaki2015time,donze2010robust,niu2023robust} has proposed scalar robustness metrics to capture spatial and temporal robustness. However, such methods are merely heuristics and do not provide the maximum  spatial and temporal perturbations that a signal can be perturbed with jointly. A naive alternative to construct a non-scalar robustness metric would be to consider the pair of spatial and temporal robustness values, as illustrated by the dotted black box in Figure \ref{fig:Pareto}. However, this approach would optimistically over-approximate admissible spatial and temporal perturbations. Indeed, our central claim is that multi-objective reasoning is required to quantitatively capture admissible spatiotemporal perturbations. Our contributions are the following:
\begin{itemize}
    \item We define spatiotemporal robustness (STR) as a multi-objective reasoning problem  via a partial order over spatial and temporal perturbations. We show that STR is sound, i.e., that its positivity implies task satisfaction. 
    \item We propose robust semantics for STR that are computationally more tractable than STR itself. We prove that these robust semantics remain sound while under-approximating the Pareto-optimal set defined by STR.
    \item Computing the robust semantics of an arbitrary signal is generally a nonconvex multi-objective optimization problem. For signed distance predicate functions, we provide efficient monitoring algorithms to compute and propagate predicate-level robust semantics through the specification's parse tree.
    \item We present two illustrative case studies  on an F-16 fighter jet and the Waymo Open Dataset, where we show the utility of STR and its robust semantics. 
\end{itemize}

The remainder is organized as follows.
In Section~\ref{sec:background}, we summarize background on signals and systems, preexisting but separate notions of spatial and temporal robustness, and results from multi-objective optimization.
Section~\ref{sec:STR} presents a formal definition of spatiotemporal robustness along with a set-theoretic formulation of robust semantics.
Algorithms to compute these robust semantics are provided in Section~\ref{sec:algorithms_compute_robust_semantics}, followed by numerical results and related work in Sections~\ref{sec:numerical_results} and \ref{sec:related_work}, respectively.
Section~\ref{sec:conclusion} closes the paper with a brief conclusion.

\section{Background}
\label{sec:background}

\subsection{Signals and Systems}
We consider $n$-dimensional discrete-time signals  $x\colon\Z\to \R^n$, where $\Z$ denotes the set of all integers.\footnote{We note that the definition of STR applies directly to continuous-time signals. However, we restrict ourselves to discrete-time signals here for computational reasons.} We index the components of the signal $x$ at time $t\in \Z$ as $x(t)=[x_1(t),\hdots,x_n(t)]\in\R^n$. Let $\deltax\colon\Z\to \R^n$ be a spatial perturbation. We let $x_\deltax\colon\Z\to \R^n$ be a \emph{spatially perturbed} signal, for which it holds that $x_\deltax(t)=x(t)+\deltax(t)$ for all $t\in \Z$. Furthermore, let $\deltat\in \N^n$ be a temporal perturbation. Here, we let $x_\deltat\colon\Z\to \R^n$ be a \emph{temporally perturbed} signal, for which it holds that $x_\deltat(t):=[x_1(t-\deltati{1}),\hdots,x_n(t-\deltati{n})]$ for all $t\in \Z$. Finally, we let $x_{\deltax,\deltat}\colon\Z\to \R^n$ be a \emph{spatiotemporally perturbed} signal, for which it holds that $x_{\deltax,\deltat}(t)=[x_1(t-\deltati{1}),\hdots,x_n(t-\deltati{n})]+\deltax(t)$ for all $t\in \Z$.\footnote{By convention, we perturb the signal $x$ temporally before we perturb it spatially. A reversed ordering yields a different but equally valid definition; the current choice is appropriate when spatial noise is applied after temporal delay. The resulting STR curves differ but both remain sound representations.} 
To keep notation simple, we write $[a;b]:=[a,b]\cap\Z$ to denote integer-valued sets.

\begin{figure}
    \centering
    \includegraphics[width=1\linewidth, alt={The fighter jet's path passes through a collision corridor and climbs above a no-fly area. The spatial and temporal perturbations are indicated through transparent overlays of the nominal path.}]{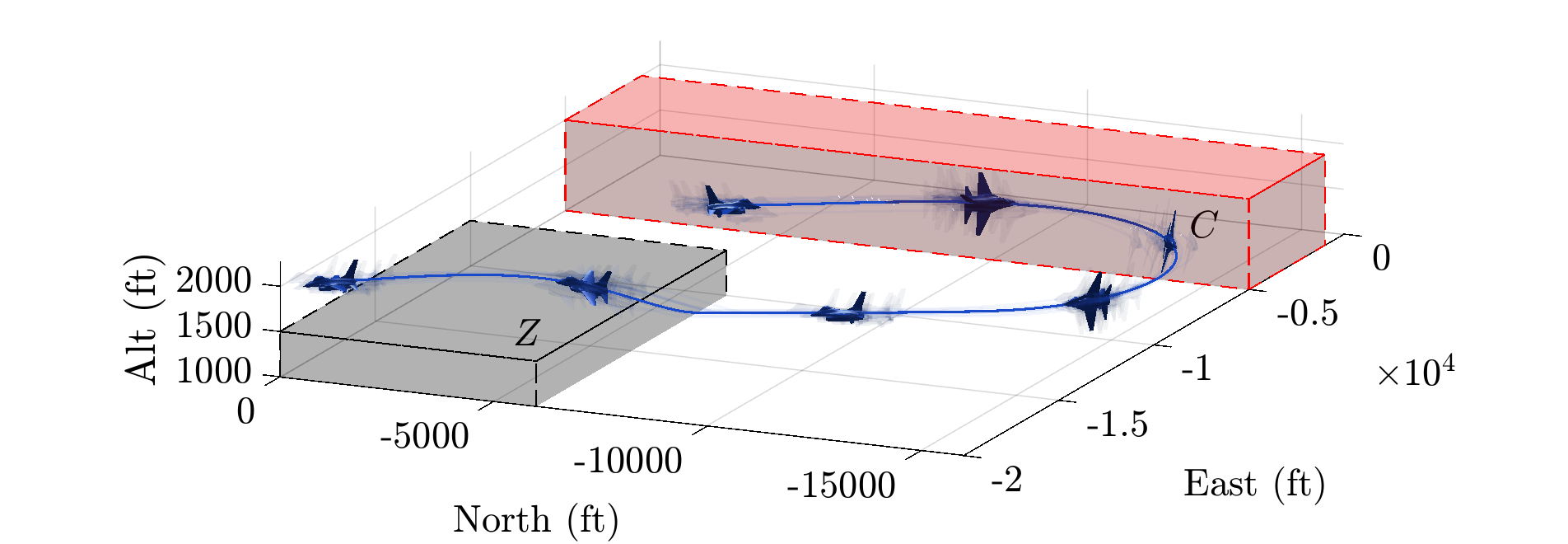}
    \caption{F-16 fighter jet flight path (nominal and 10 perturbed paths), no-fly zone $Z$ (black area), and collision corridor $C$ (red area).}
    \label{fig:flypath_simple}
\end{figure}

\begin{runexample}
    Consider a flight path $x_{\mathrm{raw}}\colon[-50;1847]\to\R^3$ of an F-16 fighter jet, obtained from the \textsc{AeroBenchVV} simulator \cite{heidlauf2018verification}, comprising the jet's relative position (in eastern and northern direction) and its altitude.
    We select and plot the sub-path $x\colon[0;1847]\to\R^3$ in Fig.~\ref{fig:flypath_simple}, allowing us to probe the signal with temporal perturbations $\deltat \in [-\Deltat;\Deltat]^3$ up to $\Deltat=50$.
    Apart from the nominal flight path (in blue), we plot 10 perturbed paths $x_{\deltax,\deltat}$ (in transparent blue) with $\deltax$ zero-mean Gaussian with variance $[100, 100, 10]$ and $\deltat$ drawn uniformly from $[-\Deltat;\Deltat]^3$.
    Here, temporal perturbations $\deltat$ account for asynchronous time shifts between the signal components, capturing effects such as sensing delays or clock misalignment, while spatial perturbations $\deltax$ correspond to deviations in the measured state, e.g., due to measurement uncertainty.
\end{runexample}

The definition of the signals $x$, $x_\deltax$, $x_\deltat$, and $x_{\deltax,\deltat}$ is generic and can capture the behavior of a broad range of systems, e.g., that of nonlinear, hybrid, and stochastic dynamical systems. We adopt the following viewpoint throughout this paper: the signal $x$ is known and nominal in the sense that it may correspond to the nominal behavior of a  system. For instance, during the design process of a system, we may have access to a model of the system, such as a simulator or a mathematical description of the system. Based on this model, we may design the system's control software, producing a nominal behavior  $x$ such that a desired  system behavior is achieved. In practice, however, the actual behavior of the system may be different from the nominal behavior $x$. Indeed, the nominal signal $x$ may be spatially, temporally, or spatiotemporally perturbed.

\subsection{Spatial and Temporal Robustness}
\label{sec:spat_temp}

\paragraph{System Specification.} Assume that we are given a desired  requirement $\phi$ that is defined over the signal $x$. For instance, $\phi$ could encode permissible driving behavior for an autonomous car \cite{hekmatnejad2019encoding}, or it could encode collision avoidance and reachability requirements for a mobile robot \cite{kress2009temporal,ahn2022can}. We let $(x,t)\models \phi$ denote the satisfaction of the requirement $\phi$ by the signal $x$ at time $t$. Such requirements --- referred to as \emph{specifications} --- could be expressed as natural language instructions, using large language models as an interface. However, they could also be expressed in a mathematical logic. While our definition of spatiotemporal robustness applies broadly to any such property, we will later adopt temporal logic as a principled means to capture system specifications. 

\paragraph{Spatial Robustness.} Now, how can we capture the satisfaction of the requirement $\phi$ in the presence of spatial perturbations? Spatial robustness --- originally presented in  \cite{fainekos2009robustness} --- provides an answer to this question and is defined as\footnote{We provide a slightly modified version from \cite{fainekos2009robustness}, where $\mathcal{SR}^\phi({x},t):= \inf_{{x}^*\in  \mathcal{L}_{\neg\phi},t} \sup_{t\ge 0}\|x(t)-x^*(t)\|$, with $\mathcal{L}_{\neg\phi,t}:=\{x\colon\Z\to\R^n \mid (x,t)\models \neg \phi\}$.}
 \begin{align}\label{eq:SR}
	\mathcal{SR}^\phi(x,t):= \sup\left\lbrace \Deltax \in\overline{\R}_{\geq0} \;\Big\vert\;   (x_\deltax, t)\models \phi,\;\forall 
    % \deltax\text{ s.t. }
    \norm{\deltax}_\infty\leq\Deltax \right\rbrace,
\end{align}
with  $\overline{\R}_{\geq0}:=\R_{\geq0}\cup \{\infty\}$ denoting the extended, non-negative real numbers and $\norm{\deltax}_\infty:= \sup_{t\in\Z} \|\deltax(t)\|$ denoting the $\infty$-signal norm where $\|\cdot\|$ can be any vector norm, e.g., the Euclidean norm.
If the signal $x$ satisfies $\phi$ at time $t$, the spatial robustness $\mathcal{SR}^\phi({x},t)$ provides a bound on the maximal spatial perturbation that can be added to $x$ without violating  $\phi$. Thus, by definition of $\mathcal{SR}^\phi({x},t)$ it immediately follows that for all perturbation levels $\Deltax\in \mathcal{SR}^\phi({x},t)$ we have that $(x_\deltax,t)\models \phi$ for all $\Deltax$-bounded disturbance signals $\deltax\colon\Z\to\R^n$. 
% $\deltax\in \mathcal{L}_{\infty,\Deltax} $.

\paragraph{Temporal Robustness.} On the other hand, temporal robustness --- originally presented in \cite{donze2010robust,lindemann2022temporal} --- can be defined as
\begin{align}\label{eq:TR}
		\mathcal{TR}^\phi({x},t) := \sup\left\lbrace \Deltat\in\overline{\Z}_{\geq0} \;\Big\vert\; (x_\deltat,t)\models \phi,\;\forall \deltat \in [-\Deltat;\Deltat]^n\right\rbrace,
\end{align}
with  $\overline{\Z}_{\geq0}:=\Z_{\geq0}\cup \{\infty\}$ denoting the extended set of non-negative integers. If the signal $x$ satisfies $\phi$ at time $t$, the temporal robustness $\mathcal{TR}^\phi({x},t)$ defines the maximum temporal perturbation $\deltat$ by which we can shift $x$ in time without violating  $\phi$. Thus, by definition of $\mathcal{TR}^\phi({x},t)$ it follows that for all perturbation levels $\Deltat\in \mathcal{TR}^\phi({x},t)$ we have that $(x_\deltat,t)\models \phi$ for all $\deltat\in [-\Deltat;\Deltat]^n$. This definition allows to asynchronously shift the signal $x$ by the vector $\deltat$, a desired property in multi-agent systems where  agents may have individual delays.

\begin{remark}[Time as Another Spatial Dimension]
    One may wonder whether treating time as an additional spatial dimension would suffice. However, encoding a temporal shift as a spatial perturbation yields a magnitude that depends on the signal itself: shifting a slowly-varying signal in time barely changes its value, while the opposite may happen for rapidly-oscillating signals. This signal-dependence makes it impossible to define a fixed, signal-independent perturbation budget, which is precisely why temporal and spatial perturbations must be treated as separate objectives in a multi-objective framework.
\end{remark}

\subsection{Multi-Objective Optimization}
\label{sec:background_multi-objective}

Optimization problems can involve the simultaneous optimization over multiple --- often conflicting --- objectives. Such problems are naturally formulated as multi-objective optimization problems. In this paper, we will encounter bi-objective optimization problems
\begin{equation}
\max_{\theta \in \Theta} \; \big(f_1(\theta), f_2(\theta)\big),
\label{eq:biobj}
\end{equation}
where $\theta \in \R^p$ denotes the decision variable, $\Theta \subseteq \R^p$ is the feasible set (typically defined by equality and inequality constraints), and $f_i \colon\R^p \to \R$ for $i \in \{1,2\}$ are two objective functions, later corresponding to spatial and temporal perturbations. In contrast to single-objective optimization, multi-objective problems generally do not admit a unique optimal solution; improving one objective may lead to degradation of another objective. A central solution concept in multi-objective optimization is \emph{Pareto optimality}, which characterizes the set of optimal trade-offs among the objectives \cite{ehrgott2005multicriteria}. A feasible point $\theta^\star \in \Theta$ is called \emph{Pareto optimal} if there exists no other $\theta \in \Theta$ such that
\[
f_i(\theta) \ge f_i(\theta^\star) \; \text{for all } i \in \{1,2\}  \text{ and }  f_j(\theta) > f_j(\theta^\star) \; \text{for some } j \in \{1,2\}.
\]
The set of all Pareto optimal points is referred to as the Pareto front. Several  techniques have been proposed to compute Pareto-optimal solutions \cite{miettinen1999nonlinear}. Among them, the \emph{$\varepsilon$-constraint method} is particularly attractive due to its simplicity and its ability to recover Pareto-optimal solutions even when the Pareto front is non-convex \cite{ehrgott2005multicriteria}.  The $\varepsilon$-constraint method proceeds by selecting one objective to optimize directly while converting the other objective into a constraint, e.g., via
\begin{equation}
\begin{aligned}
\max_{\theta \in \Theta} \quad & f_1(\theta) &
\text{s.t.} \quad & f_2(\theta) \ge \varepsilon,
\end{aligned}
\label{eq:epsilon_constraint}
\end{equation}
where $\varepsilon \in \R$ is a user-specified parameter. Intuitively, the parameter $\varepsilon$ encodes an admissible performance level for the constrained objective and allows the decision maker to explore trade-offs between objectives in a systematic manner.  By solving~\eqref{eq:epsilon_constraint} for different values of $\varepsilon$, one obtains a set of solutions that, under mild regularity assumptions, are Pareto optimal for the original multi-objective problem~\eqref{eq:biobj}; see \cite[Part~II]{miettinen1999nonlinear}. Generally, i.e., without making any such regularity assumptions, the $\varepsilon$-constraint method recovers all weakly Pareto optimal solutions \cite[Theorem~3.2.1]{miettinen1999nonlinear}, where a feasible point $\theta^\star \in \Theta$ is called \emph{weakly Pareto optimal} if there exists no other $\theta \in \Theta$ such that
$f_i(\theta) > f_i(\theta^\star)$ for all $i \in \{1,2\}$.

\section{Spatiotemporal Robustness}\label{sec:STR}

While multiple variations of $\mathcal{SR}^\phi({x})$ and $\mathcal{TR}^\phi({x})$ have appeared \cite{gilpin2020smooth,mehdipour2019arithmetic,mehdipour2024generalized,haghighi2019control,akazaki2015time}, their drawback is immediate: they are scalars that do not capture spatiotemporal perturbations, not even when considering the tuple $(\mathcal{SR}^\phi({x}),\mathcal{TR}^\phi({x}))$; see Fig.~\ref{fig:Pareto}. Spatiotemporal robustness requires {multi-objective reasoning via Pareto optimal sets that capture trade-offs between spatial and temporal robustness}. 

\subsection{Spatiotemporal Robustness as Multi-Objective Reasoning}
\label{sec:def_spat}
To present our definition of spatiotemporal robustness, we first need to introduce a partial order $\preccurlyeq$ on $\overline{\R}_{\geq0}\times \overline{\Z}_{\geq0}$, for which we consider the product order. 

\begin{definition}[Partial Order]
\label{def:par_order}
    Let $(a,b),(a',b')\in D$ be two points within a set $D\subseteq \overline{\R}_{\geq0}\times \overline{\Z}_{\geq0}$. Let $\succcurlyeq$ be a \emph{partial order} such that $(a,b) \succcurlyeq (a',b')$  if and only if $a\ge a'$ and $b\ge b'$. Furthermore, let $\succ$ be a \emph{strict partial order} such that $(a,b) \succ (a',b')$  if and only if $(a,b) \succcurlyeq (a',b')$ and  $(a,b) \neq (a',b')$. Then, we say that the point $(a,b)\in D$ is \emph{maximal} if there exists no other point $(a',b')\in D$ such that $(a',b') \succ (a,b)$.\footnote{If the set $D$ is empty, then $\max(D)$ is empty, i.e., $D=\emptyset$ implies $\max(D)=\emptyset$. } 
\end{definition}

We have purposefully presented the definitions of spatial and temporal robustness as a maximization problem, thereby revealing  structural similarity. This motivates the definition of spatiotemporal robustness, albeit different using maximal points instead to account for the multi-objective nature of the definition.

\begin{definition}[Spatiotemporal Robustness (STR)]\label{def:str}
Let $x\colon\Z\to \R^n$ be a signal, $\phi$ be a specification defined over $x$, and $t\in\Z$ be a time. Then, we define  \emph{spatiotemporal robustness} as
\begin{align}\label{eq:STR}
    \mathcal{STR}^\phi({x},t):= \max\left(D^\phi(x,t)\right)\subseteq \overline{\R}_{\geq0}\times \overline{\Z}_{\geq0}.
\end{align}
where the $\max(\cdotx)$ operator computes all maximal points over the domain
\begin{equation*}
    D^\phi(x,t):=\left\lbrace (\Deltax,\Deltat)
    % \in\overline{\R}_{\geq0}\times \overline{\Z}_{\geq0} 
    \;\Big\vert\; (x_{\deltax,\deltat},t) \models \phi,\;
    \forall\norm{\deltax}_\infty\leq\Deltax,\forall\deltat\in[-\Deltat;\Deltat]^n  \right\rbrace,%\label{eq:feasible_domain}
\end{equation*}
which captures all spatiotemporal perturbation levels up to which we can perturb the signal $x$ without changing the satisfaction of $\phi$ at time $t$.
\end{definition}

Note that the maximal points of the set $D^\phi(x,t)$ are equivalent to the Pareto optimal points of a multi-objective optimization problem with decision variables $\theta:=[\Deltax,\Deltat]$, feasible set $\Theta:=D^\phi(x,t)$, and objective functions $f_1(\theta):=\Deltax$ and $f_2(\theta):=\Deltat$.    In other words, for a point in $\mathcal{STR}^\phi({x},t)$ there does not exist any other point in $D^\phi(x,t)$ that increases spatial or temporal robustness without decreasing the other. The spatiotemporal robustness $\mathcal{STR}^\phi({x},t)$, consisting of all Pareto optimal points, therefore defines a Pareto optimal set that reveals important trade-offs, recall Figure~\ref{fig:motivating_example}. The intuition  is as follows: given a tuple $(\Deltax,\Deltat)\in \mathcal{STR}^\phi({x},t)$, we know that the signal $x$ can simultaneously sustain  spatial and temporal perturbations at  levels of  $\Deltax$ and $\Deltat$, respectively. The next result trivially follows by definition of spatiotemporal robustness. 
\begin{corollary}[STR Implies Robust Task Satisfaction]
	\label{cor:1}
    Let $x\colon\Z\to \R^n$ be a signal, $\phi$ be a specification defined over $x$, and $t\in\Z$ be a time. Assume further that $(x,t)\models \phi$ holds. For each joint perturbation level $(\Deltax,\Deltat)\in \mathcal{STR}^\phi({x},t)$, it holds that $(x_{\deltax,\deltat},t)\models \phi$ for all spatiotemporal perturbations $(\deltax,\deltat)$ such that $\norm{\deltax}_\infty\leq\Deltax$ and $\deltat\in[-\Deltat;\Deltat]^n$.
\end{corollary}

%\textcolor{red}{Can we reason about some monotonicity property of spatiotemporal robustness (and the robust semantics presented later)?}

\begin{runexample}
    Suppose we want to determine a joint perturbation level $(\Deltax,\Deltat)\in\overline{\R}_{\geq0}\times \overline{\Z}_{\geq0}$ under which the fighter jet in Fig.~\ref{fig:flypath_simple} successfully steers clear of the no-fly zone $Z\subseteq\R^3$ (in black). 
    This amounts to computing the STR of a simple avoidance specification $\phi_{\mathrm{avoid}}$, which is satisfied if $x_{\deltax,\deltat}(t)\notin Z$ for all timesteps $t\in[0;1847]$ and disturbances $\norm{\deltax}_\infty\leq\Deltax$ and $\deltat\in[-\Deltat;\Deltat]^3$; then, $(\Deltax,\Deltat)\in \mathcal{STR}^{\phi_{\mathrm{avoid}}}({x},0)$. We plot the STR of $\phi_{\mathrm{avoid}}$ --- i.e., the set of all such Pareto optimal points $(\Deltax,\Deltat)$ --- in Fig.~\ref{fig:F16_STR} (details in Section \ref{sec:numerical_results}). The remainder of the paper focuses on how this set can be computed efficiently.
\end{runexample}

\subsection{Spatiotemporal Robustness for Temporal Logic Specifications}
The Pareto optimal set $\mathcal{STR}^\phi({x},t)$ can be computed by solving the aforementioned multi-objective optimization problem. However, this turns out to be computationally challenging in full generality, i.e., for arbitrary specifications $\phi$, due to the search over signals $\deltax$ satisfying $\norm{\deltax}_\infty\leq\Deltax$. To obtain a tractable algorithm for the computation of $\mathcal{STR}^\phi({x},t)$, we use \emph{signal temporal logic} (STL) specifications and their inherent structure \cite{maler2004monitoring}, see also \cite{bartocci2018specification,lindemann2025formal} for surveys. STL builds on atomic elements, so called predicates. The truth value of a predicate $\mu\colon\R^n\to\mathbb{B}$ is obtained after evaluation of a  function $h^\mu\colon\R^n\to\R$, which we refer to as the predicate function. In other words, we have
\begin{align}\label{eq:STL_predicate}
\mu({x}):=\begin{cases}
\top & \text{if } h^\mu({x})\ge 0\\
\bot &\text{otherwise.}
\end{cases}
\end{align} 
Building on top of predicates, the STL syntax\footnote{The later defined robust semantics for $\mathcal{STR}^\phi({x},t)$ would require careful attention for negation operators. Instead, we consider the STL syntax to be in positive normal form, meaning that $\phi$ does not contain  negations, by replacing the negation operator with disjunction, eventually, and always operators. 
Note that this also applies to the release operator in discrete time, which can be recursively defined: 
$\phi \,\mathbf{R}_{[a,b]}\, \psi = (\mathbf{G}_{[a,b]}\,\psi) \vee (\vee_{k=a}^{b}(\mathbf{G}_{[a,k]}\,\psi \wedge \mathbf{F}_{[k,k]}\,\phi))$.
This is without loss of generality as every STL task $\phi$ can be written as an STL task in positive normal form  \cite{sadraddini2015robust}.} is expressed as
\begin{align}\label{eq:full_STL}
\phi \; ::= \; \top \; | \; \mu  \; | \; \phi' \wedge \phi'' \; | \; \phi' \vee \phi'' \; | \;  \G_I \phi' \; | \; \F_I \phi' \; | \; \phi'  \U_I \phi''  \,
\end{align}
where $\phi'$ and $\phi''$ are already STL formulas and  $\top$, $\wedge$, $\vee$, $\G_I$, $\F_I$, and $\U_I$ with $I\subseteq \Z_{\ge 0}$ denote the true symbol, conjunction, disjunction, always, eventually, and until operators, respectively. The STL semantics are defined as a relation $\models$ between the signal $x$ and the STL formula $\phi$, and $(x, t)\models \phi$ indicates that the signal $x$ satisfies the formula $\phi$
at time $t$, while $(x, t)\not\models \phi$ indicates the opposite.

\begin{definition}[STL Qualitative Semantics]\label{def:qualitative_semantics}
	Let $x\colon\Z\to \R^n$ be a signal, $\phi$ be an STL specification, and $t\in\Z$ be a time. The \emph{semantics} of $\phi$ are recursively defined over its structure as
	\begin{align*}
	 (x,t) \models
	 \top  &\;\;\;\text{ iff } \;\;\;	\text{holds by definition}, \\
	 ({x},t) \models
	 \mu   &\;\;\;\text{ iff }\;\;\;	h^\mu({x}(t))\ge 0\\
	 ({x},t) \models \phi' \wedge \phi'' &\;\;\;\text{ iff }\;\;\; ({x},t) \models \phi' \text{ and } ({x},t) \models \phi''\\
     ({x},t) \models \phi' \vee \phi'' &\;\;\;\text{ iff }\;\;\; ({x},t) \models \phi' \text{ or } ({x},t) \models \phi''\\
     ({x},t) \models  \G_I  \phi' &\;\;\;\text{ iff }\;\;\; \forall t' \in t\oplus I \colon ({x},t')\models \phi' \\
	 ({x},t) \models \F_I \phi' &\;\;\;\text{ iff }\;\;\; \exists t' \in t\oplus I \text{ s.t. }({x},t')\models \phi' \\
	 ({x},t) \models \phi' \U_I \phi'' &\;\;\;\text{ iff }\;\;\; \exists t'' \in t\oplus I \text{ s.t. }({x},t'')\models \phi'' \\
	 & \;\;\;\;\;\;\;\;\;\;\, \text{ and } \forall t'\in [t;t'']\colon\, ({x},t') \models \phi'.
	\end{align*}
\end{definition}

\begin{runexample}
    The specification for avoiding the no-fly zone $Z$ in Fig.~\ref{fig:flypath_simple} can be written as $\phi_{\mathrm{avoid}} := \G_{[0;1847]}(x(t)\notin Z)$, where $\mu_{\mathrm{avoid}}(x(t)):=(x(t)\notin Z)$ acts as the predicate. 
    This can be easily expressed in the form \eqref{eq:STL_predicate} by selecting the predicate function $h^{\phi_{\mathrm{avoid}}}(x(t)):=\sd(x(t), Z)$ via the \emph{signed distance function}
    \begin{equation}
        \sd(z, S) := \begin{cases}
            -\inf_{z'\in\partial S} \norm{z - z'}, &\text{if } z\in S\\
            \phantom{-}\inf_{z'\in\partial S} \norm{z - z'}, &\text{if } z\not\in S\\
            \phantom{-}0, &\text{if } z\in\partial S
        \end{cases}\label{eq:signed_distance_fcn}
    \end{equation}
    defined for a point $z\in\R^n$ and set $S\subseteq\R^n$, where $\partial S$ is the boundary of $S$.
\end{runexample}

For an STL specification $\phi$, it is well established that the computation of the spatial and temporal robustness $\mathcal{SR}^\phi({x},t)$ and $\mathcal{TR}^\phi({x},t)$ as defined in Section~\ref{sec:spat_temp}  can be challenging. Indeed, their computation is intractable for most STL specifications as it involves computing the satisfaction of $\phi$ over sets of signals. This has motivated more tractable notions like the robust semantics  \cite{fainekos2009robustness,donze2010robust}, which are computed recursively over the structure of $\phi$ and provide an efficient under-approximation.  We provide their definition for $\mathcal{SR}^\phi({x},t)$ in Appendix~\ref{app:robust_semantics}, while their definition for $\mathcal{TR}^\phi({x},t)$ follows almost analogously \cite{donze2010robust}.  The aforementioned computational challenges remain for the spatiotemporal robustness $\mathcal{STR}^\phi({x},t)$, motivating us next to define robust semantics for $\mathcal{STR}^\phi({x},t)$.

\subsection{Robust Semantics for Spatiotemporal Robustness}

It is  worth recalling the close connection between logic and set theory. As such, the $\forall$ and $\exists$ operators in Definition \ref{def:qualitative_semantics} can be interpreted as union and intersection operations over sets. To define robust semantics for $\mathcal{STR}^\phi({x},t)$, we will therefore define how we can combine multiple Pareto optimal sets. Specifically, we have to define the minimum and maximum operators for Pareto optimal sets, which will later correspond to performing conjunction and disjunction operations.
\begin{definition}[Minimum and Maximum]\label{def:min_max}
    Let $D^1,D^2\in\R^2$ be two sets; potentially Pareto optimal. The \emph{downward closure} of set $D^i$, $i\in \{1,2\}$ is
    \begin{align*}
        D^i_\downarrow \!:=\!\!\left\lbrace(\Deltax',\Deltat')\in \overline{\R}_{\geq0}\!\times\! \overline{\Z}_{\geq0} \;\Big\vert\; \exists (\Deltax,\Deltat) \in D^i \!\text{ s.t. } (\Deltax,\Deltat) \!\succcurlyeq\! (\Deltax',\Deltat') \right\rbrace\!.
    \end{align*}
    With this, we define the \emph{minimum} (resp. \emph{maximum}) of $D^1$ and $D^2$ as $\min\{D^1,\allowbreak D^2\}\allowbreak:=\allowbreak\max\{D^1_\downarrow \cap D^2_\downarrow\}$ (resp. $\max\{D^1,D^2\}:=\max\{D^1 \cup D^2\}$). 
\end{definition}

\begin{figure}
\centering
    \def\svgwidth{.8\linewidth}
    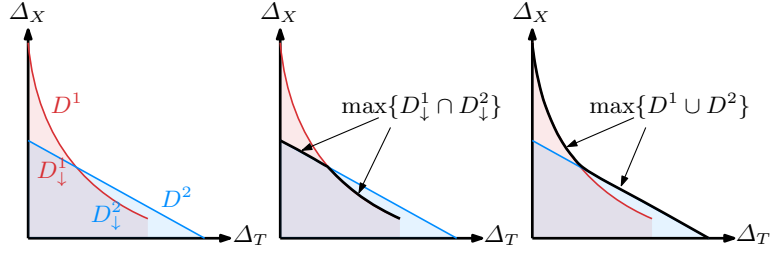

    \caption{Left panel: Two Pareto optimal sets $D^1$ (red line) and $D^2$ (blue line) and their downward closures $D^1_\downarrow$ (red shaded area) and $D^2_\downarrow$ (blue shaded area). Middle panel: The minimum of $D^1$ and $D^2$, i.e., $\min\{D^1,D^2\}:=\max\{D^1_\downarrow \cap D^2_\downarrow\}$. Right panel: The maximum  of $D^1$ and $D^2$, i.e., $\max\{D^1,D^2\}:=\max\{D^1 \cup D^2\}$.}
    %alt={Three panels showing the minimum and maximum of two Pareto optimal sets. The left panel shows two intersecting Pareto optimal sets D1 (red line) and D2 (blue line) along with their downward closures D1_downarrow (red shaded area) and D2_downarrow (blue shaded area). The middle panel shows the minimum of D1 and D2 as a black line. The right panel shows the maximum of D1 and D2 as a black line.}
    \label{fig:downward_closure}
\end{figure}

Geometrically, $\min\{D^1,D^2\}$ corresponds to the \emph{lower envelope} of the  Pareto optimal sets $D^1$ and $D^2$, while $\max\{D^1,D^2\}$ corresponds to their \emph{upper envelope}. The downward closure is used for the definition of $\min\{D^1,D^2\}$, while it is not needed for the definition of $\max\{D^1,D^2\}$. To illustrate why, consider Fig.~\ref{fig:downward_closure}. The left panel shows two intersecting Pareto optimal sets $D^1$ (red line) and $D^2$ (blue line) along with their downward closures $D^1_\downarrow$ (red shaded area) and $D^2_\downarrow$ (blue shaded area). The middle panel shows the minimum of $D^1$ and $D^2$ as a black line. Note that this minimum can only be obtained by computing the maximal points of $D^1_\downarrow \cap D^2_\downarrow$, while using $D^1 \cap D^2$ would not give the correct result. The right panel shows the maximum of $D^1$ and $D^2$,  which can be  obtained directly by computing the maximal points of $D^1 \cup D^2$. 
These definitions extend to more than two sets, i.e., for a countable set $I$ we define $\min_{t\in I} D^t:=\min\{D^{t_1},\min\{D^{t_2}, \hdots\}\}$ and $\max_{t\in I} D^t:=\max\{D^{t_1},\max\{D^{t_2}, \hdots\}\}$ for indices $t_1,t_2,\hdots\in I$. 

We are now ready to define the robust semantics for $\mathcal{STR}^\phi({x},t)$, which structurally resemble those for $\mathcal{SR}^\phi({x},t)$ and $\mathcal{TR}^\phi({x},t)$, but are defined for Pareto optimal sets and require using the max and min operators from Definition~\ref{def:min_max}.
\begin{definition}[Robust Semantics]\label{def:robust_semantics}
	Let  $x\colon\Z\to \R^n$ be a signal, $\phi$ be an STL specification from \eqref{eq:full_STL}, and $t\in\Z$ be a time. The \emph{robust semantics} of $\phi$ are recursively defined over its structure as
	\begin{align*}
		\rho^{\top}({x},t)& := \left\lbrace(\infty,\infty)\right\rbrace,\\
		\rho^{\mu}({x},t)& :=\mathcal{STR}^\mu({x},t),\\
		\rho^{\phi' \wedge \phi''}({x},t) &:= 	\min\left\lbrace \rho^{\phi'}({x},t),\; \rho^{\phi''}({x},t)\right\rbrace\\
        \rho^{\phi' \vee \phi''}({x},t) &:= 	\max\left\lbrace \rho^{\phi'}({x},t),\; \rho^{\phi''}({x},t)\right\rbrace\\
        \rho^{G_I \phi'}({x},t) &:= \min_{t'\in t\oplus I}  \rho^{\phi'}({x},t')\\
        \rho^{F_I \phi'}({x},t) &:= \max_{t'\in t\oplus I}  \rho^{\phi'}({x},t')\\
		\rho^{\phi' U_I \phi''}({x},t) &:= \max_{t''\in t\oplus I}  \min\left\lbrace \rho^{\phi''}({x},t''),\; \min_{t'\in [t;t'']}\rho^{\phi'}({x},t') \right\rbrace.
	\end{align*}
\end{definition}

As we show in the next section, computing the robust semantics $\rho^{\phi}({x},t)$ is  principled and computationally more tractable than computing the spatiotemporal robustness $\mathcal{STR}^\phi({x},t)$. Fortunately, the robust semantics $\rho^{\phi}({x},t)$ are an under-approximation of the spatiotemporal robustness $\mathcal{STR}^\phi({x},t)$; the following theorem formalizes this, the proof of which has been relegated to Appendix~\ref{app:proof_robust_semantics_lowerbound_str}.

\begin{theorem}[Robust Semantics Under-Approximate STR]
\label{thm:1}
    Let $x\colon\allowbreak\Z\allowbreak\to \R^n$ be a signal, $\phi$ be an STL specification from \eqref{eq:full_STL}, and $t\in\Z$ be a time. Assume further that $(x,t)\models \phi$ holds. Then, it holds that $\rho^{\phi}({x},t)$ does not lie above $\mathcal{STR}^\phi({x},t)$. Formally, this means that for each point $(\Deltax,\Deltat)\in \rho^{\phi}({x},t)$ there exists a point $(\Deltax',\Deltat')\in  \mathcal{STR}^\phi({x},t)$ such that $ (\Deltax',\Deltat') \succcurlyeq (\Deltax,\Deltat)$, or equivalently that $\mathcal{STR}^\phi_\downarrow({x},t) \supseteq \rho^{\phi}_\downarrow({x},t)$.
\end{theorem}

Theorem~\ref{thm:1} implies that we can use $\rho^{\phi}({x},t)$ instead of  $\mathcal{STR}^\phi({x},t)$ to reason about spatiotemporal robustness. Indeed, given a point $(\Deltax,\Deltat)\in \rho^{\phi}({x},t)$, Theorem~\ref{thm:1} tells us that there exists a point $(\Deltax',\Deltat')\in  \mathcal{STR}^\phi({x},t)$ that is at least as robust as $(\Deltax,\Deltat)$, i.e., such that $ (\Deltax',\Deltat') \succcurlyeq (\Deltax,\Deltat)$. From Corollary~\ref{cor:1}, it follows that the specification $\phi$ is still satisfied by the signal $x$ under spatiotemporal perturbations at level $(\Deltax,\Deltat)$, summarized next.

\begin{corollary}[Robust Semantics Capture Robust Task Satisfaction]
	\label{cor:2}
	Let $x\colon\Z\to \R^n$ be a signal, $\phi$ be an STL specification from \eqref{eq:full_STL}, and $t\in\Z$ be a time. Assume further that $(x,t)\models \phi$ holds.  For each perturbation level $(\Deltax,\Deltat)\in \rho^{\phi}({x},t)$, it holds that $(x_{\deltax,\deltat},t)\models \phi$ for all spatiotemporal perturbations $(\deltax,\deltat)$ such that $\norm{\deltax}_\infty\leq\Deltax$ and $\deltat\in[-\Deltat;\Deltat]^n$.
\end{corollary}

In summary, Theorem \ref{thm:1} and Corollary \ref{cor:2} tell us that we can use the robust semantics $\rho^\phi(x,t)$ to reason over spatiotemporal perturbations. However, compared to the spatiotemporal robustness $\mathcal{STR}^\phi({x},t)$, the robust semantics are more conservative in the sense that they may indicate a smaller set of permissible spatiotemporal perturbations. The upside of using the robust semantics is their computational tractability, as we show in the next section. 

While we assumed in Theorem~\ref{thm:1} and Corollary \ref{cor:2}  that $\phi$ is satisfied by the nominal signal $x$, i.e., that $(x,t)\models \phi$ holds, it follows by their definitions that both the spatiotemporal robustness $\mathcal{STR}^\phi({x},t)$ and the robust semantics $\rho^{\phi}({x},t)$ allow us to reason about the qualitative satisfaction of $\phi$.

\begin{corollary}[Soundness]
\label{lem:1}
    Let $x\colon\Z\to \R^n$ be a signal, $\phi$ be an STL specification from \eqref{eq:full_STL}, and $t\in\Z$ be a time. Then, it holds that
\begin{align*}
    \mathcal{STR}^\phi({x},t)\neq \emptyset \;\;\; &\implies \;\;\; (x,t)\models \phi,\\
     \rho^{\phi}({x},t)\neq \emptyset \;\;\; &\implies \;\;\; (x,t)\models \phi.
\end{align*}
\end{corollary}

\section{Monitoring Algorithms for the Robust Semantics}\label{sec:algorithms_compute_robust_semantics}

In the previous section, we have introduced the robust semantics $\rho^\phi(x,t)$ as an alternative for the spatiotemporal robustness $\mathcal{STR}^\phi({x},t)$. While we showed that $\rho^\phi(x,t)$ can be a conservative alternative, we indicated already that  $\rho^\phi(x,t)$ may generally be easier to compute than $\mathcal{STR}^\phi({x},t)$. Indeed, taking a look at Definition \ref{def:robust_semantics} reveals that \emph{computing the robust semantics $\rho^\phi(x,t)$ for an STL specification $\phi$ only requires computing the robust semantics $\rho^\mu(x,t)$ for predicates $\mu$}, which is equivalent to computing the spatiotemporal robustness $\mathcal{STR}^\mu({x},t)$, followed by recursively applying min and max operations. The core challenge therefore reduces to computing the  spatiotemporal robustness $\rho^\mu(x,t)\equiv\mathcal{STR}^\mu({x},t)$ for predicates $\mu$. We present algorithms to compute $\rho^\mu(x,t)$ in Section \ref{sec:comp_predicate}, while we summarize the overall algorithm to compute  $\rho^\phi(x,t)$ in Section \ref{sec:comp_overall}.  

\begin{runexample}
    Besides the safety specification $\phi_{\mathrm{avoid}}$ defined earlier, consider two additional requirements:
    \begin{itemize}
        \item $\phi_{\mathrm{threat}}:=\G_{\geq 581} \left(\sd(x,C) > 0\right)$: The jet must clear the threat corridor $C\subset\R^3$ (red box in Fig.~\ref{fig:flypath_simple}) by time $t=581$;
        \item $\phi_{\mathrm{climb}}:=\G_{\{1349\}}\left(\left(1600-x_3> 0\right) \wedge \F_{[0;300]} \left(x_3-1800> 0\right)\right)$: Between $t=1349$ and $t=1649$, the jet must climb from an altitude below 1600~ft to above 1800~ft.
    \end{itemize}

    Consistent with the setting of Theorem~\ref{thm:1} and Corollary \ref{cor:2}, the jet's flight path $x$ satisfies the overall specification $\phi:=\phi_{\mathrm{avoid}} \wedge \phi_{\mathrm{threat}} \wedge \phi_{\mathrm{climb}}$, visualized in the parse tree in Fig.~\ref{fig:spec_parsetree_fighterjet}.
    The following sections provide algorithms for computing the robust semantics $\rho^\phi(x,t)$ bottom-up, starting at the predicate-level leaf nodes of the parse tree (highlighted nodes in Fig.~\ref{fig:spec_parsetree_fighterjet}), and propagating Pareto optimal sets upward through application of min/max operations (recall Definition~\ref{def:min_max}).
    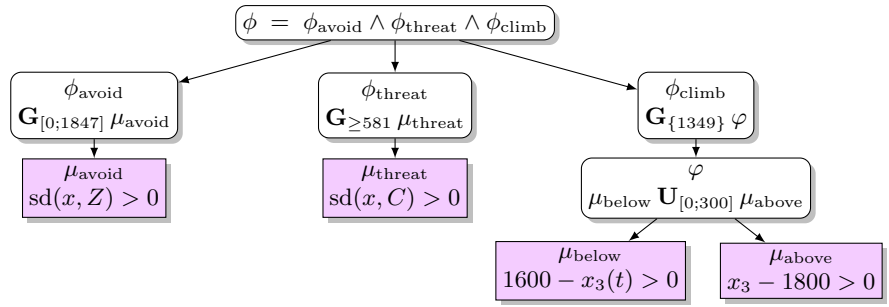
\begin{figure}
        \centering
        \begin{tikzpicture}[
          scale=1,
          level distance=11mm,
          sibling distance=40mm,
          every node/.style={draw, fill=white, rounded corners, align=center, inner sep=2.5pt, drop shadow},
          edge from parent/.style={draw,-latex}
        ]
        \node {$\phi \;=\; \phi_{\mathrm{avoid}} \wedge \phi_{\mathrm{threat}} \wedge \phi_{\mathrm{climb}}$}
          child { node {$\phi_{\mathrm{avoid}}$\\$\G_{[0;1847]}\, \mu_{\mathrm{avoid}}$}
            child { node[fill=accent!20!white, sharp corners] {$\mu_{\mathrm{avoid}}$\\$\sd(x,Z)>0$} }
          }
          child { node {$\phi_{\mathrm{threat}}$\\$\G_{\geq 581}\, \mu_{\mathrm{threat}}$}
            child { node[fill=accent!20!white, sharp corners] {$\mu_{\mathrm{threat}}$\\$\sd(x,C)>0$} }
          }
          child { node {$\phi_{\mathrm{climb}}$\\$ \G_{\{1349\}}\,\varphi$}
            child { node {$\varphi$\\$\mu_{\mathrm{below}}\wedge \F_{[0;300]}\,\mu_{\mathrm{above}}$}
            % tighten ONLY these two children:
            child[sibling distance=28mm] { node[fill=accent!20!white, sharp corners] {$\mu_{\mathrm{below}}$\\$1600-x_3(t)>0$} }
            child[sibling distance=28mm] { node[fill=accent!20!white, sharp corners] {$\mu_{\mathrm{above}}$\\$x_3-1800>0$} }
            }
          };
        \end{tikzpicture}
        \caption{Parse tree of the fighter-jet benchmark specification.}
        %alt={A parse tree of the fighter-jet benchmark specification. The root node is labeled with the overall specification phi, which is a conjunction of three sub-specifications: phi_avoid, phi_threat, and phi_climb. Each sub-specification has its own child nodes representing the predicates involved in the specification. The leaf-level predicate nodes are highlighted in a different color.}
        \label{fig:spec_parsetree_fighterjet}
    \end{figure}
\end{runexample}

\subsection{Computing the Robust Semantics of Predicates}
\label{sec:comp_predicate}

It can be seen that $\rho^\mu(x,t)$ --- which was defined as $\mathcal{STR}^\mu({x},t)= \max\left(D^\mu(x,t)\right)$  --- is equivalent to the Pareto optimal points of a multi-objective optimization problem with decision variables $\theta:=[\Deltax,\Deltat]$, feasible set $\Theta:=D^\mu(x,t)$, and objective functions $f_1(\theta):=\Deltax$ and $f_2(\theta):=\Deltat$. The next result presents an equivalent formulation  of this  optimization problem by re-writing the feasible set $D^\mu(x,t)$, the proof of which has been relegated to Appendix~\ref{app:proof_equivalent_representation}.

\begin{theorem}[Equivalent Representation for $\boldsymbol{\rho^\mu(x,t)}$]
\label{thm:2}
    Let $x\colon\Z\allowbreak\to \R^n$ be a signal, $\phi$ an STL specification from \eqref{eq:full_STL}, and $t\in\Z$ a time. Then, $\rho^\mu(x,t)$ is equivalent to the Pareto optimal points of the following multi-objective optimization problem
    \begin{subequations}
\begin{align}
     \max_{\Deltax,\Deltat}\;&  (\Deltax,\Deltat),\label{eq:MO_prob_objfcn} \\
\text{s.t.}\; &
  g_{x,t}^\mu(\Deltax,\Deltat)\ge 0,\label{eq:MO_prob_constr}
\end{align}\label{eq:MO_prob}%
\end{subequations}
where the robust predicate margin $g_{x,t}^\mu(\Deltax,\Deltat)$ is defined as
\begin{equation}
    g_{x,t}^\mu(\Deltax,\Deltat) := \inf_{\substack{\norm{\overline{\deltax}}\leq\Deltax\\\deltat\in[-\Deltat;\Deltat]^n}} 
    h^\mu \!\left(\phantom{\begin{bmatrix}
        x_1(t-\deltati{1})\\\vdots\\x_n(t-\deltati{n})
    \end{bmatrix}}\right.\hspace{-6.5em}
        \underbrace{\begin{bmatrix}
        x_1(t-\deltati{1})\\\vdots\\x_n(t-\deltati{n})
    \end{bmatrix} + \overline{\deltax}}_{{x_{\overline{\deltax},\deltat}(t)}}
    \hspace{-6.5em}\left.\phantom{\begin{bmatrix}
        x_1(t-\deltati{1})\\\vdots\\x_n(t-\deltati{n})
    \end{bmatrix}}\right)\label{eq:inner_program}
\end{equation}
with $\overline{\deltax}\in\R^n$ being a vector, denoting momentary spatial perturbations.
\end{theorem}

The function $g_{x,t}^\mu(\Deltax,\Deltat)$ in \eqref{eq:inner_program} captures the worst-case margin of satisfying the predicate $\mu$ at state $x$ and time $t$ under joint spatial and temporal perturbations of magnitudes $\Deltax$ and $\Deltat$. The significance of Theorem~\ref{thm:2} is that it reduces evaluating $(\Deltax,\Deltat)\in D^\mu(x,t)$, which requires analyzing functions $\deltax\colon\Z\to\R^n$ in the function space $\norm{\deltax}_\infty\leq\Deltax$, to satisfaction of the constraint $g_{x,t}^\mu(\Deltax,\Deltat)\ge 0$, which only requires  analyzing  spatial perturbations $\overline{\deltax}\in\R^n$ in the vector space $\|\overline{\deltax}\|\le \Deltax$, thereby significantly reducing the computational complexity.\footnote{It becomes evident here why we use the signal space $\norm{\deltax}_\infty= \sup_{t\in\Z} \|\deltax(t)\|\le \Deltax$ via the  $\infty$-signal norm. Other signal norms do not permit rewriting $D^\mu(x,t)$.} 
Lastly, since enlarging $\Deltax$ or $\Deltat$  can only decrease the worst-case margin, the function $g_{x,t}^\mu(\Deltax,\Deltat)$ is monotonically non-increasing in $(\Deltax,\Deltat)$.

% \begin{remark}[Pareto Front Characterization]
%     Since enlarging $\Deltax$ or $\Deltax$  can only decrease the worst-case margin, the function $g_{x,t}^\mu(\Deltax,\Deltat)$ is monotonically non-increasing in $(\Deltax,\Deltat)$. 
%     If the predicate function $h^\mu$ is continuous in $\overline{\deltax}$,
%     % If the predicate $\mu$ and the signal $x$ are continuous, then the function
%     % $g_{x,t}^\mu(\Deltax,\Deltat)$ is also continuous in
%     % $(\Deltax,\Deltat)$ \Lars{We have a problem here, $x$ is discrete-time and cannot really be continuous}. In this case, 
%     a feasible point
%     $(\Deltax,\Deltat)\in D^\mu(x,t)$ is Pareto optimal if and only if
%     $g_{x,t}^\mu(\Deltax,\Deltat)=0$, i.e., it lies on the boundary of
%     $D^\mu(x,t)$. Consequently, the Pareto optimal set of
%     \eqref{eq:MO_prob} coincides with the zero-level set of
%     $g_{x,t}^\mu(\Deltax,\Deltat)$.
% \end{remark}

To compute this Pareto optimal set in practice, we use the $\varepsilon$-constraint scalarization method discussed  in Section~\ref{sec:background_multi-objective}. We also remark that the multi-objective optimization problem in \eqref{eq:MO_prob} is convex only if the predicate constraint $h^\mu(x_{\overline{\deltax},\deltat}(t)) \geq 0$ is concave in $\deltat$ and $\overline{\deltax}$, which will not hold in general.

\paragraph{A) General Predicate Functions.}
Applying the $\varepsilon$-constraint method, as summarized in \eqref{eq:epsilon_constraint}, to the multi-objective optimization problem in \eqref{eq:MO_prob} yields a family of single-objective optimization problems  parametrized by $\varepsilon\in\N_{\geq0}$:
\begin{align*}
\begin{split}
    \max_{\Deltax,\Deltat}\;&  \Deltax,\\
    \text{s.t.}\; &
    g_{x,t}^\mu(\Deltax,\Deltat)\ge 0\\
    & \Deltat \geq \varepsilon.
\end{split}
\end{align*}
For general (e.g., non-convex) predicate functions $h^\mu$, this optimization problem is still a bilevel optimization problem since $g_{x,t}^\mu(\Deltax,\Deltat)$ is an optimization problem in itself. We therefore have to solve the above problem as a robust optimization problem. As $g_{x,t}^\mu(\Deltax,\Deltat)$ is monotonically non-increasing in $(\Deltax,\Deltat)$, we can reduce  the complexity and simply enforce $\Deltat=\varepsilon$ instead of the inequality $\Deltat\ge\varepsilon$ and iterate through all $\Deltat\in\overline{\Z}_{\geq0}$. Thereby, we are guaranteed to obtain all weakly Pareto optimal points \cite[Theorem 3.2.1]{miettinen1999nonlinear}. Since \eqref{eq:MO_prob} has only two objective functions, it is straightforward to recover the Pareto optimal points from these weakly Pareto optimal points (more info below).

Note that any solution procedure iterating through $\Deltat=\varepsilon$ inflates the set of admissible temporal disturbances $\deltat\in[-\Deltat;\Deltat]^n$ in each iteration. Indeed, in each iteration we have to solve the following robust optimization problem
\begin{align}
\begin{split}
    \mathrm{SR}(\deltat) := \max_{\Deltax}\;&  \Deltax,\\
    \text{s.t.}\; &
    \inf_{\norm{\overline{\deltax}}\leq\Deltax} h^\mu(x_{\overline{\deltax},\deltat}(t)) \geq 0
\end{split}\label{eq:SO_prog_general}
\end{align}
for each $\deltat\in[-\Deltat;\Deltat]^n$ and then take the minimum over all these $\mathrm{SR}(\deltat)$.
We summarize our implementation in Algorithm~\ref{alg:compute_STR_pred}, which returns the corresponding robust semantics $\rho^\mu(x,t)$ via its functional envelope $\Deltax(\Deltat)$. 
\begin{algorithm}
    \caption{Predicate-level $\rho^{\textcolor{black}\mu}(x,t)$ for general predicate functions}
    \label{alg:compute_STR_pred}
    \begin{algorithmic}[1]
        \Require Predicate function $h^\mu\colon\R^n\to\R$, signal $x\colon\Z\to\R^n$, time index $t\in\Z$
        \Ensure Envelope representation $\rho^{\mu}(x,t)=\{(\Deltax(\Deltat),\Deltat)\}_{\Deltat\in\N_{\geq0}}$
        \For{$\Deltat = 0,1,2,\ldots,\Deltat^{\max}$} \Comment{increasing temporal disturbance level}
            \State $\mathcal{T}_{\Deltat}\gets [-\Deltat;\Deltat]^n\setminus [-\Deltat+1;\Deltat-1]^n$ \Comment{temporal shell}
            \State $\mathrm{SR}^\star\gets
            +\infty$  
            \Comment{initialize temporary variable}
            \color{black!40!white}\ForAll{$\deltat \in \mathcal{T}_{\Deltat}$} \Comment{parallelizable}
                    \State $\mathrm{SR}\gets  \mathrm{SR}(\deltat)$ \Comment{solve robust program \eqref{eq:SO_prog_general}}
                    \If{$\mathrm{SR} < 0$}
                        \State 
                        \textbf{return} $\Deltax$ 
                        \Comment{not $\Deltat$-temporally robust}
                    \Else
                        \State $\mathrm{SR}^\star\gets \min\{\mathrm{SR}^\star,\,\mathrm{SR}\}$
                    \EndIf
            \EndFor
            \color{black}\If{$\Deltat = 0$}
                \State $\Deltax(\Deltat)\gets
                \mathrm{SR}^\star$
            \Else
            \State $\Deltax(\Deltat)\gets
                \min\{\mathrm{SR}^\star,\,\Deltax(\Deltat-1)\}$ \Comment{monotone envelope}
                
            \EndIf
        \EndFor
    \end{algorithmic}
\end{algorithm}
We note that Lines~2 and 10--13 reduce the complexity of a naive algorithm that would check all $\deltat\in[-\Deltat;\Deltat]^n$ in every iteration by instead only reasoning over the newly introduced temporal perturbations $[-\Deltat;\Deltat]^n\setminus [-\Deltat+1;\Deltat-1]^n$.
We also note  that $\Deltax(\cdotx)$ in Line~7 is always well defined, i.e.,  $\Delta_X(\Delta_T)$ is non-empty for at least $\Delta_T=0$ due to our assumption that $(x,t)\models \phi$. If one is interested in recovering the Pareto optimal points from the weakly Pareto optimal points, we can simply remove $\Deltax(\Deltat-1)$ if $\Deltax(\Deltat-1)=\Deltax(\Deltat)$ after Line 13.

\medskip

\paragraph{B) Signed Distance Predicate Functions.}

A computational bottleneck of  Algorithm~\ref{alg:compute_STR_pred} is in Line~5 where we have to solve the optimization problem \eqref{eq:SO_prog_general}. This raises the question under what circumstances  \eqref{eq:SO_prog_general} can be computed efficiently.
One popular type of predicate function that fits the bill is the signed distance function. Let us hence consider the case when the predicate function $h^\mu$ is given by the signed Euclidean distance function $\sd(\cdotx,S)$ from \eqref{eq:signed_distance_fcn}, where $S\subseteq\R^n$ is a closed set. Our approach is motivated by the next result.

\begin{lemma}\label{lem:neg_sdf_ball}
    Let $S\subseteq\R^n$ be closed and nonempty, and let $\sd(\cdot,S)$ denote the
    signed distance function as in~\eqref{eq:signed_distance_fcn}.
    Then for any $z\in\R^n$ and $\Delta\ge 0$,
    \begin{align}
    \inf_{\|\delta\|\le \Delta}\sd(z-\delta,S) &\ge \sd(z,S)-\Delta. \label{eq:inf_ineq}
    \end{align}
\end{lemma}

\begin{proof}

By \cite[Cor.~1]{balint2019operations}, every signed distance function is
$1$-Lipschitz, i.e., such that
\[
|\sd(z,S)-\sd(z',S)|\le \|z-z'\|\qquad \forall z,z'\in\R^n.
\]
If we fix $\delta$ and set $z':=z-\delta$, we obtain   $\sd(z-\delta,S)\ge \sd(z,S)-\|\delta\|$.
Taking $\inf_{\|\delta\|\le\Delta}$ over this inequality directly yields \eqref{eq:inf_ineq}.\qed

\end{proof}

In fact, it is possible to show that the bound in \eqref{eq:inf_ineq} is tight if the set $S$ is convex, i.e., for the lower bound \eqref{eq:inf_ineq} if follows that $\inf_{\norm{{\delta}}\leq\Delta} \sd(z-{\delta},S) = \sd(z,S) - \Delta$. We skip the rather elaborate proof here. 

We can now modify Algorithm~\ref{alg:compute_STR_pred} using Lemma~\ref{lem:neg_sdf_ball} to simplify computation. All we have to do is to change Line~5 to $\mathrm{SR}\gets  \sd(x_\deltat,S)$; the rest remains identical.

\subsection{Computing the Robust Semantics of Signal Temporal Logic}\label{sec:comp_overall}
Based on the predicate-level robust semantics $\rho^\mu(x,t)$ computed in the previous subsection, the robust semantics $\rho^\phi(x,t)$ of the specification $\phi$ can be computed by propagating the Pareto optimal sets through the specification's parse tree (recall Fig.~\ref{fig:spec_parsetree_fighterjet} for an example), by applying min/max set operations as summarized in Definition~\ref{def:robust_semantics}.
We provide a compact version including only two possible operators in Algorithm~\ref{alg:compute_STR_spec_brief} to provide some intuition to the reader, while we include a full version including all operators from Definition~\ref{def:robust_semantics} in Appendix~\ref{app:comp_overall}.

\begin{algorithm}
    \caption{Compute specification-level $\rho^\phi(x,t)$ (compact version)}
    \label{alg:compute_STR_spec_brief}
    \begin{algorithmic}[1]
    \Require STL-like specification $\phi$, signal $x\colon\Z\to\R^n$, time $t\in\Z$, $\Deltat^{\max}\in\N$
    \Ensure Envelope $\rho^\phi(x,t)=\{(\Deltax_\phi(\Deltat),\Deltat)\}_{\Deltat\in\N_{\geq0}}$
    
    \State $\mathcal{N}\gets \phi.\texttt{nodes\_postorder()}$ \Comment{bottom-up traversal}

    \ForAll{$v \in \mathcal{N}$} 
        \color{black!40!white}\If{$v$ is a predicate node with predicate $\mu_v$}
            \For{$t' \in t\oplus v.\texttt{horizon}$} \Comment{only where needed upstream}
                \State $\Deltax_v(\cdotx;t') \gets 
                \texttt{Algorithm\_\ref{alg:compute_STR_pred}}(\mu_v,x,t')$
            \EndFor
    
        \color{black}\ElsIf{$v$ is Boolean $\wedge$ with children $c_1,c_2$}
            \For{$t' \in t\oplus v.\texttt{horizon}$}
                \For{$\Deltat=0,\ldots,\Deltat^{\max}$}
                    \State $\Deltax_v(\Deltat;t') \gets 
                    \min\{\Deltax_{c_1}(\Deltat;t'),\,\Deltax_{c_2}(\Deltat;t')\}$
                \EndFor
            \EndFor
    
        \color{black!40!white}\ElsIf{$v$ is $\G_{[a;b]}$ with child $c$}
            \For{$t' \in t\oplus v.\texttt{horizon}$}
                \For{$\Deltat=0,\ldots,\Deltat^{\max}$}
                    \State $\Deltax_v(\Deltat;t') \gets 
                    \min\limits_{k\in[a;b]}\ \Deltax_{c}(\Deltat;t'+k)$
                \EndFor
            \EndFor
        \color{black}\EndIf $\vdots$
    \EndFor
    \State \Return $\Deltax^\star_{\phi}(\cdotx;t)$
    \end{algorithmic}
\end{algorithm}

\begin{remark}[Scalability]
    Our STR monitor mirrors standard STL robustness monitoring: scalar robustness values at the predicate level are replaced by Pareto front sets, which propagate through the same parse tree structure via min/max operators overloaded via downward closure theory (Definition~\ref{def:min_max}). The additional complexity over standard monitoring therefore arises from two sources only: (i) computing a Pareto front instead of a scalar at each predicate, and (ii) set-valued min/max operations instead of scalar ones. The latter is lightweight, as evidenced by the runtimes reported in our experiments in Section~\ref{sec:numerical_results}. Predicate-level complexity scales with $\Delta_T$ and signal dimension $n$, since the algorithm iterates over shifts in $[-\Delta_T, \Delta_T]^n$; per-shift cost depends on the predicate type, ranging from closed-form evaluation to solving a numerical optimization problem.
\end{remark}

\begin{remark}[Coupled Time Disturbances] Variables can be coupled by setting $\deltat_i = \deltat_j$ for some $i \neq j$, restricting the disturbance set to a lower-dimensional subspace of $[-\Deltat, \Deltat]^n$. This has two advantages as it (1) reduces computational complexity and (2) yields more truthful robustness estimates. We note that this topic was also discussed in Remark~2 of the paper \cite{lindemann2022temporal}.
\end{remark}

\section{Case Studies: F-16 Fighter Jet \& Robotaxi}\label{sec:numerical_results}

In this section, we provide two case studies that illustrate the advantages of spatiotemporal robustness reasoning. We implemented the presented algorithms in MATLAB and ran all experiments on an Apple M1 Pro MacBook.\footnote{Code available at \url{https://doi.org/10.24433/CO.5553407.v1} (interactive) and \url{https://doi.org/10.5281/zenodo.19843304} (archived).}
For signed-distance predicates, computations reduce entirely to min/max operations over Pareto fronts; no numerical solver is invoked, so results are exact up to floating-point arithmetic.

\subsection{Fighter Jet}
Recall the running example of the F-16 fighter jet with specification $\phi:=\phi_{\mathrm{avoid}} \wedge \phi_{\mathrm{threat}} \wedge \phi_{\mathrm{climb}}$.
We use Algorithm~\ref{alg:compute_STR_spec} to compute the robust semantics $\rho^\phi(x,0)$ and report alongside it the intermediate robust semantics $\rho^{\phi_{\mathrm{avoid}}}( x,0)$, $\rho^{\phi_{\mathrm{threat}}}( x,0)$, and $\rho^{\phi_{\mathrm{climb}}}( x,0)$ in Fig.~\ref{fig:F16_STR}.
The computation of $\rho^{\phi_{\mathrm{avoid}}}( x,0)$, $\rho^{\phi_{\mathrm{threat}}}( x,0)$, and $\rho^{\phi_{\mathrm{climb}}}( x,0)$ took 53.23, 2.91, and 13.62 seconds, respectively, and the computation of $\rho^\phi(x,0)$ an additional 0.0001 seconds.\footnote{For the avoidance specification ${\phi_{\mathrm{avoid}}}$, we truncate the horizon to $t\geq1500$.}
% Threat corridor: 2.912514 seconds
% Avoid obstacle: Took 53.227962 seconds. (truncating before $t=1500$)
% Climb: 13.617337 seconds (for below and above predicates) + 0.010985 seconds.
% Putting all together: 0.000097 seconds.
In contrast to classical robust semantics --- which would return a single scalar spatial robustness margin by fixing $\Deltat=0$ --- the set-valued envelope of the robust semantics reveal how admissible spatial perturbations degrade when temporal perturbations increase.

Notably, the subformula limiting the signal's spatial robustness changes as the temporal disturbance budget increases.
For $\Deltat<30$, the spatiotemporal robustness is dominated by the climb requirement $\phi_{\mathrm{climb}}$, indicating sensitivity to temporal misalignment during the altitude transition.
Beyond $\Deltat\geq30$, the limiting factor becomes $\phi_{\mathrm{threat}}$, reflecting a qualitative shift where temporal delays primarily threaten timely clearance of the threat corridor. This example illustrates how spatiotemporal robustness can disentangle which sub-specifications fail first and why, which is not possible through scalar robustness alone.

\begin{figure}
    \centering
    \includegraphics[width=1\linewidth, alt={Four spatiotemporal robustness curves for the fighter jet example: three for different sub-specifications and one for the overall specification. They are all monotonically decreasing and the overall curve is the minimum of the three.}]{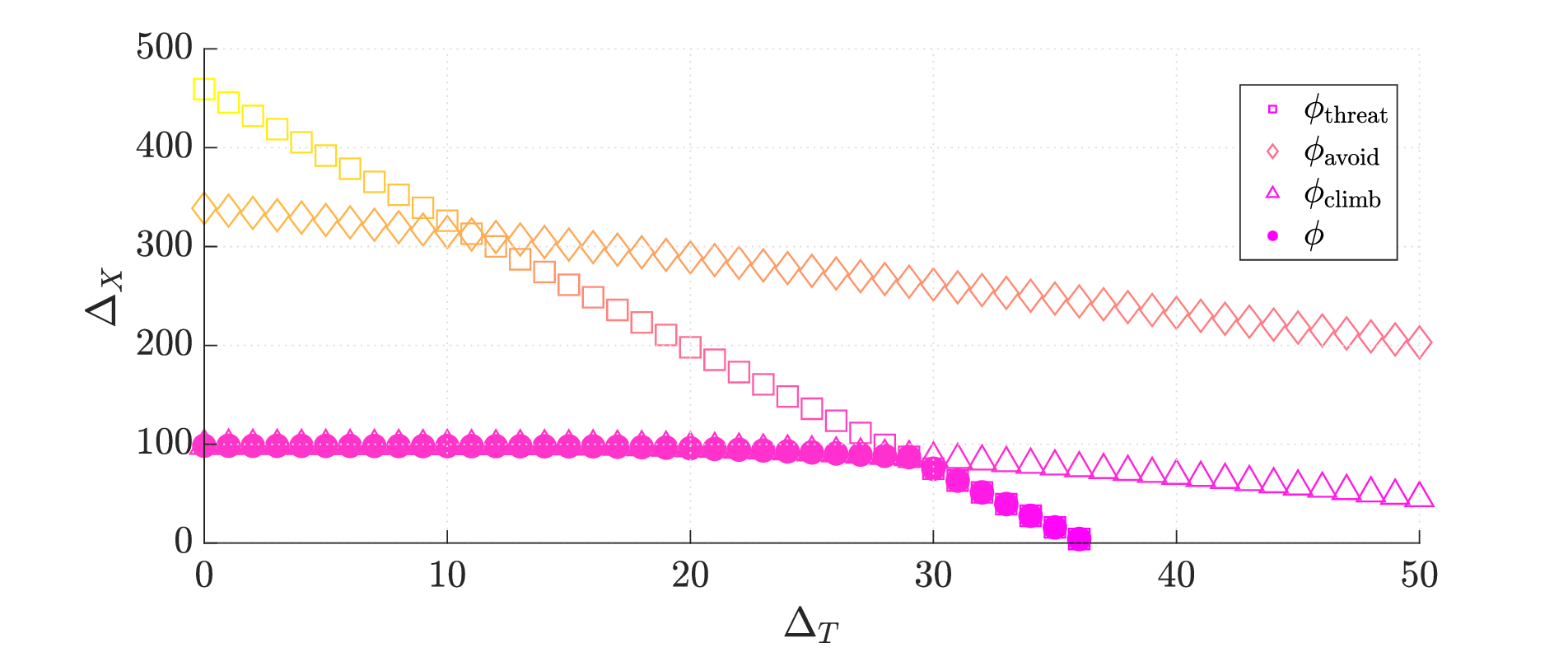}
    \caption{Spatiotemporal robustness of the fighter jet example.}
    \label{fig:F16_STR}
\end{figure}

\subsection{Robotaxi}
We now turn to a real-world dataset. In particular, we consider a motion scenario from the Waymo Open Dataset\footnote{Available at \url{https://waymo.com/open/}.} involving a vehicle passing a crossroad only moments before a pedestrian crosses the street (see Fig.~\ref{fig:waymo_setup}).
We wish to analyze \emph{how robust} the behavior of the signal $[x_v,x_p]\colon[0;90]\to\R^4$ is, where $x_v$ and $x_p$ encode the positions of the interacting vehicle and pedestrian, respectively. 

To this end, we first compute the STR with respect to a collision avoidance specification $\phi=\G_{[0;90]}\, (\norm{x_v-x_p}\geq d_{\mathrm{min}})$, where $d_{\mathrm{min}}=1$ is a minimum safety distance. This yields the (essentially) linear robust semantics curve in Fig.~\ref{fig:waymo_str}. This case is interesting as it indicates a linear relationship between spatial and temporal robustness, which we do in general not expect. Interestingly, it turns out that the slope of the robust semantics is not proportional to the relative velocity between vehicle and pedestrian, which one could suspect.

For another case, we fix the position of the pedestrian on the sidewalk (black star enclosed by $d_{\mathrm{min}}$ safety circumference in Fig.~\ref{fig:waymo_setup}) and compute the spatiotemporal robustness for the signal $x_v\colon[0;90]\to\R^2$. The resulting spatiotemporal robustness plot is also shown in Fig.~\ref{fig:waymo_str}, which is now nonlinear. The computational time for the experiments was 0.67 and 0.03 seconds, respectively.

\begin{figure}[t]
    \centering
    \begin{subfigure}[b]{.49\linewidth}
        \centering
        \includegraphics[width=\textwidth, alt={Setup for the robotaxi example. Dashed lines indicate the boundaries of the road and surrounding environment. Colored paths represent the trajectories of vehicles and pedestrians. A black star indicates the fixed position of the pedestrian for the second experiment.}]{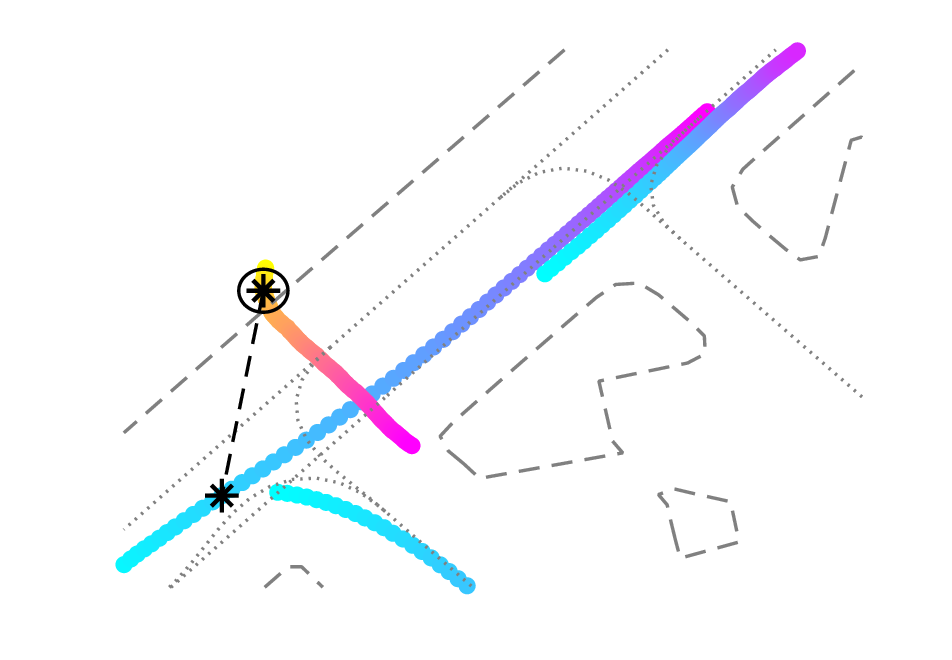}
        \caption{Paths of vehicles and pedestrians.}
        \label{fig:waymo_setup}
    \end{subfigure}
    \begin{subfigure}[b]{.49\linewidth}
        \centering
        \includegraphics[width=\textwidth, alt={Two spatiotemporal robustness curves for the robotaxi example. Both of them are monotonically decreasing. The upper one is for both vehicle and pedestrian, while the lower one is for the vehicle only. The lower curve is always below the upper one and decreases approximately linearly. The upper one plateaus for increasing temporal disturbances.}]{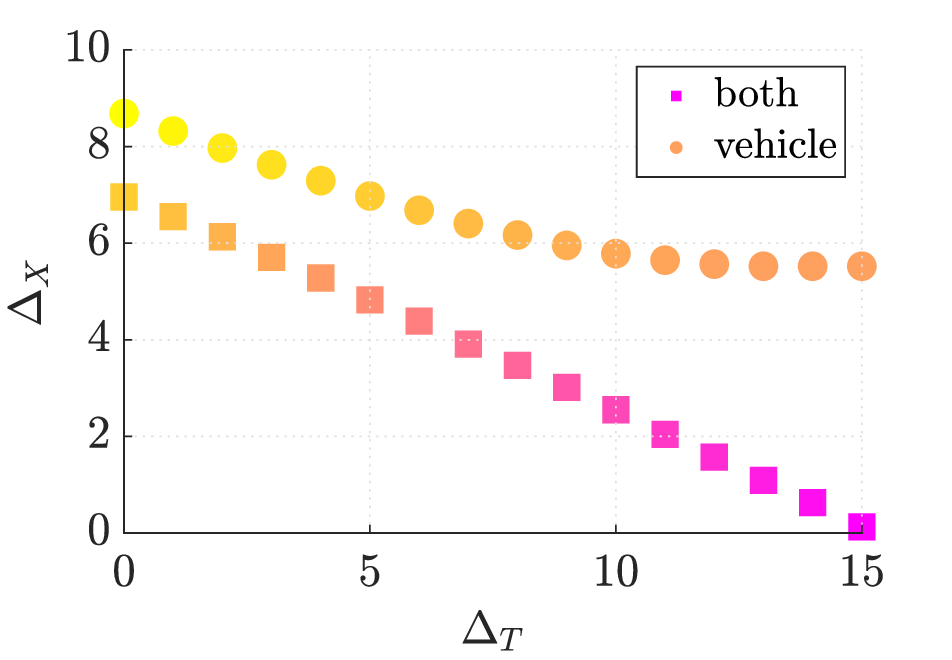}
        \caption{Spatiotemporal robustness.}
        \label{fig:waymo_str}
    \end{subfigure}
    \caption{Spatiotemporal robustness of the robotaxi example.}
    \label{fig:waymo}
\end{figure}

\section{Related Work}\label{sec:related_work}
Robustness and the ability to deal with uncertainties lies at the core of any properly engineered system. Existing literature on robust control \cite{doyle1988state,zhou1998essentials,freeman2008robust} and time-delay systems \cite{zhong2006robust,fridman2014introduction,niculescu2002delay} analyze and mitigate the effects of uncertainty, but focus solely on simple system specifications --- mainly stability and forward invariance --- without reasoning about high-level system objectives as required in many real-world applications. Indeed, metric interval temporal logic and signal temporal logic have emerged as powerful formalisms to capture rich specifications for continuous and hybrid dynamical systems \cite{koymans1990specifying,alur1996benefits,maler2004monitoring}. The task of monitoring such high-level system specifications plays an important role due to the increasing level of automation and complexity of autonomous systems \cite{luo2024sample,lindemann2024formal,seshia2022toward,kwiatkowska2023trust}. 

Robustness of temporal logic specifications was first studied in \cite{fainekos2009robustness,donze2010robust} via robust semantics that capture the geometric distance of a signal from unsatisfiability, thereby capturing all admissible spatial signal perturbations. Subsequently, these robust semantics were used for monitoring \cite{donze2013efficient,deshmukh2017robust,dokhanchi2014line,bortolussi2023conformal,lindemann2023conformal}, falsification/testing \cite{annpureddy2011s,abbas2013probabilistic,sankaranarayanan2012falsification}, and planning/control \cite{raman2014model,vasile2017sampling,pant2017smooth,lindemann2018control}. Many variants of the robust semantics have appeared with various desirable properties \cite{haghighi2019control,mehdipour2019arithmetic,mehdipour2024generalized,varnai2020robustness,gilpin2020smooth}, such as smoothness to make them more amenable to optimization \cite{dhonthi2021study}.

The aforementioned works capture  spatial robustness, but lack the ability to reason over temporal perturbations. The authors in \cite{donze2010robust,lindemann2022temporal} propose temporal robustness via robust semantics that capture permissible time shifts of a signal, while permissible time dilations were studied in \cite{rino2024efficiently}. Robust semantics for temporal robustness were used 
within robotics/control  \cite{rodionova2021time,rodionova2022combined,chen2023stl,le2024time,buyukkocak2025resilient,wang2024synthesis}. While these works consider time shifts of a signal, time dilations were studied in \cite{rino2024efficiently}. The only works that address spatiotemporal robustness are \cite{akazaki2015time,donze2010robust,niu2023robust}, but only provide scalar metrics. We note that \cite{donze2010robust} already hints at the multi-objective structure that we developed in this paper. Indeed, \cite{donze2010robust} argues that one could (1) fix a spatial robustness level and compute the resulting temporal robustness, and (2) then sweep over all such spatial robustness levels. The authors  call the resulting tuples ``a kind of Pareto curve''. We deliver precisely what \cite{donze2010robust} observed but left open: (1) we are the first to compute STR Pareto curves at the predicate level and propagate them through the parse tree (Definition~\ref{def:robust_semantics}), yielding sound under-approximations of the true STR (Theorem~\ref{thm:1}), which would be very difficult to compute directly (Definition~\ref{def:str}); (2) we precisely define min/max operations over Pareto curves (Definition~\ref{def:min_max}), which \cite{donze2010robust} omits and whose absence can yield misleading robustness interpretations; and (3) we prove the relationship connecting the robust semantics to the true STR (Theorem~\ref{thm:1}) and connect it to multi-objective optimization for efficient computation.

Lastly, we briefly discuss related lines of work that are tangential to our focus. Conformance testing frameworks quantify the similarity --- along spatial and temporal dimensions --- of the outputs of two systems directly, and then use such similarity metrics to transfer satisfaction guarantees of temporal logic tasks between systems \cite{abbas2014formal,deshmukh2015quantifying}. Further related are works on mining parametric temporal logic formulas from data, where it was noted that the parameter mining problem can be converted into a multi-objective optimization problem~\cite{hoxha2018mining,vazquez2018time,kyriakis2019specification}.

\section{Conclusion and Future Work}\label{sec:conclusion}
We introduced spatiotemporal robustness as a principled framework for reasoning about the satisfaction of temporal logic tasks, evaluated over discrete-time signals, under joint spatial and temporal perturbations. We further developed robust semantics that under-approximate spatiotemporal robustness, along with tractable monitoring algorithms that use ideas from multi-objective optimization. This work lays the foundations for future advances in the verification and design of safety-critical autonomous systems under multi-dimensional uncertainty. For future work, we aim to incorporate dynamical systems properties into the definition and computation of spatiotemporal robustness. We also intend to focus on algorithmic improvements for computing the robust semantics, so they can be used in downstream tasks such as verification and control synthesis.

\begin{credits}
    \subsubsection*{\discintname}
    The authors have no competing interests to declare that are relevant to the content of this article.

    \subsubsection*{Data-Availability Statement.}
    The artifact (incl. data and code) that supports the experimental findings of this paper is openly available at \url{https://doi.org/10.5281/zenodo.19843304}.
\end{credits}

 \bibliographystyle{splncs04}
 \bibliography{bib}
 
 \appendix

\section{Robust Semantics for Spatial Robustness}\label{app:robust_semantics}

For the convenience of the reader, we recall the robust semantics for spatial robustness from \cite{fainekos2009robustness} which will structurally resemble the robust semantics for spatiotemporal robustness.
\begin{definition}[STL Robust Semantics]\label{def:robust_semantics_}
	Let $x\colon\Z\to \R^n$ be a signal, $\phi$ be an STL specification, and $t\in\Z$ be a time. The \emph{semantics} of $\phi$ are recursively defined over its structure as 
	\begin{align*}
		\rho^{\top}({x},t)& := \infty,\\
		\rho^{\mu}({x},t)& := \sd(x(t),O^\mu),\\
		\rho^{\phi' \wedge \phi''}({x},t) &:= 	\min\left\lbrace \rho^{\phi'}({x},t),\, \rho^{\phi''}({x},t)\right\rbrace,\\
        \rho^{\phi' \vee \phi''}({x},t) &:= 	\max\left\lbrace \rho^{\phi'}({x},t),\; \rho^{\phi''}({x},t)\right\rbrace\\
        \rho^{G_I \phi'}({x},t) &:= \inf_{t'\in t\oplus I}  \rho^{\phi'}({x},t')\\
        \rho^{F_I \phi'}({x},t) &:= \sup_{t'\in t\oplus I}  \rho^{\phi'}({x},t')\\
		\rho^{\phi' \U_I \phi''}({x},t) &:= \sup_{t''\in t\oplus I}  \min\left\lbrace \rho^{\phi''}({x},t''),\, \inf_{t'\in [t;t'']}\rho^{\phi'}({x},t') \right\rbrace.
	\end{align*}
	where we define the set of states that satisfy the predicate $\mu$ as
	\begin{align*}
		O^\mu:=\left\lbrace x\in\R^n \mid h^\mu(x)\ge 0\right\rbrace
	\end{align*}
	and where the signed distance $\sd(\cdot)$ is defined as in equation \eqref{eq:signed_distance_fcn}. 
\end{definition}

The robust semantics are easier to compute and provide an underapproximation of the spatial robustness, i.e., we have that $|\rho^{\phi}({x},t)|\le |\mathcal{SR}^\phi({x},t)|$.

\section{Proof of Theorem~\ref{thm:1}}\label{app:proof_robust_semantics_lowerbound_str}
\begin{proof}
        Our proof strategy is to show that $D^\phi(x,t)\supseteq \rho^{\phi}({x},t)$. Once we have established this fact, it is trivial to conclude that $\mathcal{STR}^\phi_\downarrow({x},t) \supseteq \rho^{\phi}_\downarrow({x},t)$ since $\mathcal{STR}^\phi({x},t)=\max(D^\phi(x,t))$. We proceed recursively by structural induction over the syntax of STL specifications from \eqref{eq:full_STL}, excluding negation operators.

        \smallskip
        
        \emph{True $\top$: } For the Boolean truth value $\top$, it is straightforward to compute the domain   $D^\top(x,t)=\overline{\R}_{\geq0}\times \overline{\Z}_{\geq0}=[0,\infty]\times \{0,\hdots,\infty\}$. By definition, the robust semantics are $\rho^{\top}({x},t)=\{(\infty,\infty)\}$, so that it follows that $D^\top(x,t)\supseteq \rho^{\top}({x},t)$.

        \smallskip
        
        \emph{Predicate $\mu$: } For predicates $\mu$, it is trivial to see that $D^\mu(x,t)\supseteq \rho^{\mu}({x},t)$ by the definition of $\rho^{\mu}({x},t)$, which is $\rho^{\mu}({x},t)=\mathcal{STR}^{\mu}({x},t)=\max(D^\mu(x,t))$. 
        
        \smallskip
        
        \emph{Conjunctions $\phi'\wedge \phi''$: } For conjunctions $\phi'\wedge \phi''$, we assume by induction that $D^{\phi'}({x},t) \supseteq \rho^{\phi'}({x},t)$ and $D^{\phi''}({x},t) \supseteq \rho^{\phi''}({x},t)$.  We note that the set $D^{\phi'\wedge \phi''}(x,t)$ can  be  decomposed and written as 
        \begin{align*}
    D^{\phi'\wedge \phi''}\!(x,t)&=\Big\lbrace (\Deltax,\Deltat)\in\overline{\R}_{\geq0}\!\times\! \overline{\Z}_{\geq0} \;\Big\vert\; (x_{\deltax,\deltat},t) \models \phi' \wedge \phi'',\;
    \forall\norm{\deltax}_\infty\leq\Deltax,\Big.\\
    &\quad\Big.\forall\deltat\in[-\Deltat;\Deltat]^n  \Big\rbrace,\\
    & =\Big\lbrace (\Deltax,\Deltat)\in\overline{\R}_{\geq0}\times \overline{\Z}_{\geq0} \;\Big\vert\;  (x_{\deltax,\deltat},t) \models \phi',\;
    \forall\norm{\deltax}_\infty\leq\Deltax,\Big.\\
    &\quad\Big.\forall\deltat\in[-\Deltat;\Deltat]^n   \Big\rbrace \cap \Big\lbrace (\Deltax,\Deltat)\in\overline{\R}_{\geq0}\!\times\! \overline{\Z}_{\geq0} \!\;\Big\vert\;\! (x_{\deltax,\deltat},t) \models \phi'',\Big.\\
    &\quad\Big.
    \forall\norm{\deltax}_\infty\leq\Deltax,\forall\deltat\in[-\Deltat;\Deltat]^n  \Big\rbrace,\\
    &=D^{\phi'}(x,t) \cap D^{\phi''}(x,t).
\end{align*}
The robust semantics $\rho^{\phi'\wedge \phi''}({x},t)$ from Definition~\ref{def:robust_semantics} can further be written as
\begin{align*}
    \rho^{\phi'\wedge \phi''}({x},t)&=\min\left\lbrace\rho^{\phi'}({x},t),\,\rho^{\phi''}({x},t)\right\rbrace=\max\left\lbrace\rho^{\phi'}_\downarrow({x},t)\cap \rho^{\phi''}_\downarrow({x},t)\right\rbrace,
\end{align*}
where we used Definition~\ref{def:min_max} for the minimum operation of $\rho^{\phi'}({x},t)$ and $\rho^{\phi''}({x},t)$. 

By the definition of $D^{\phi}({x},t)$, we have that $D^{\phi}({x},t)=D^{\phi}_\downarrow({x},t)$.\footnote{If  $(\Deltax,\Deltat)\in D^\phi(x,t)$, then any  $(\Deltax',\Deltat')\in\overline{\R}_{\geq0}\times\overline{\Z}_{\geq0}$ with $(\Deltax,\Deltat)\succcurlyeq (\Deltax',\Deltat')$ is  such that $(\Deltax',\Deltat')\in D^\phi(x,t)$, i.e., $D^{\phi}({x},t)$ is downward closed and $D^{\phi}({x},t)=D^{\phi}_\downarrow({x},t)$. This follows by definition of $D^{\phi}({x},t)$ and since $\Deltax' \le \Deltax$ and $\Deltat'\le \Deltat$. }  Therefore, our induction assumption $D^{\phi'}({x},t) \supseteq \rho^{\phi'}({x},t)$ and $D^{\phi''}({x},t) \supseteq \rho^{\phi''}({x},t)$  implies that $D^{\phi'}({x},t) \supseteq \rho^{\phi'}_\downarrow({x},t)$ and $D^{\phi''}({x},t) \supseteq \rho^{\phi''}_\downarrow({x},t)$. We hence know that $D^{\phi'}(x,t) \cap D^{\phi''}(x,t)\supseteq \rho^{\phi'}_\downarrow({x},t)\cap \rho^{\phi''}_\downarrow({x},t)$, by which it follows that
\begin{align*}
    D^{\phi'\wedge \phi''}(x,t)=D^{\phi'}(x,t) \cap D^{\phi''}(x,t)\supseteq \max(\rho^{\phi'}_\downarrow({x},t)\cap \rho^{\phi''}_\downarrow({x},t))=\rho^{\phi'\wedge \phi''}({x},t).
\end{align*}

        \smallskip
        
\emph{Disjunctions $\phi'\vee \phi''$: } The proof for disjunctions $\phi'\vee \phi''$ proceeds similarly as for conjunctions. We assume by induction that $D^{\phi'}({x},t) \supseteq \rho^{\phi'}({x},t)$ and $D^{\phi''}({x},t) \supseteq \rho^{\phi''}({x},t)$.  We note that the set $D^{\phi'\vee \phi''}(x,t)$ can  be  written as 
        \begin{align*}
    D^{\phi'\vee \phi''}\!(x,t)&=\Big\lbrace (\Deltax,\Deltat)\in\overline{\R}_{\geq0}\!\times\! \overline{\Z}_{\geq0} \;\Big\vert\; (x_{\deltax,\deltat},t) \models \phi' \vee \phi'',\;
    \forall\norm{\deltax}_\infty\leq\Deltax,\Big.\\
    &\quad\Big.\forall\deltat\in[-\Deltat;\Deltat]^n  \Big\rbrace,\\
    &{\supseteq}\Big\lbrace (\Deltax,\Deltat)\in\overline{\R}_{\geq0}\!\times\! \overline{\Z}_{\geq0} \!\;\Big\vert\;\! (x_{\deltax,\deltat},t) \models \phi',\;
    \forall\norm{\deltax}_\infty\leq\Deltax,\Big.\\
    &\quad\Big.\forall\deltat\in[-\Deltat;\Deltat]^n   \Big\rbrace \cup \Big\lbrace (\Deltax,\Deltat)\in\overline{\R}_{\geq0}\!\times\! \overline{\Z}_{\geq0} \!\;\Big\vert\;\! (x_{\deltax,\deltat},t) \models \phi'',\Big.\\
    &\quad\Big.
    \forall\norm{\deltax}_\infty\leq\Deltax,\,\forall\deltat\in[-\Deltat;\Deltat]^n  \Big\rbrace\\
    &=D^{\phi'}(x,t) \cup D^{\phi''}(x,t).
\end{align*}
The robust semantics $\rho^{\phi'\vee \phi''}({x},t)$ from Definition~\ref{def:robust_semantics} can further be written as
\begin{align*}
    \rho^{\phi'\vee \phi''}({x},t)&=\max\left\lbrace\rho^{\phi'}({x},t),\,\rho^{\phi''}({x},t)\right\rbrace=\max\left\lbrace\rho^{\phi'}({x},t)\cup \rho^{\phi''}({x},t)\right\rbrace,
\end{align*}
where we used Definition~\ref{def:min_max} for the maximum operation of $\rho^{\phi'}({x},t)$ and $\rho^{\phi''}({x},t)$. 

Our induction assumption $D^{\phi'}({x},t) \supseteq \rho^{\phi'}({x},t)$ and $D^{\phi''}({x},t) \supseteq \rho^{\phi''}({x},t)$  implies that $D^{\phi'}(x,t) \cup D^{\phi''}(x,t)\supseteq \rho^{\phi'}({x},t)\cup \rho^{\phi''}({x},t)$, by which it follows that
\begin{align*}
    D^{\phi'\vee \phi''}(x,t){\supseteq}D^{\phi'}(x,t) \cup D^{\phi''}(x,t)\supseteq \max(\rho^{\phi'}({x},t)\cup \rho^{\phi''}({x},t))=\rho^{\phi'\vee \phi''}({x},t).
\end{align*}
    
    \emph{Always $\G_I \phi'$}: For always $\G_I \phi'$, we recall the semantics $({x},t) \models  \G_I \phi'$ to be $\forall t' \in t\oplus I,\, ({x},t')\models \phi'$, while the robust semantics are $\rho^{G_I \phi'}({x},t) := \min_{t'\in t\oplus I}  \rho^{\phi'}({x},t')$. This corresponds to the case of $|t\oplus I|-1$ conjunction operators so that $D^{\G_I \phi'}(x,t)\supseteq \rho^{\G_I \phi'}({x},t)$ follows immediately.
     \smallskip

         \emph{Eventually $\F_I \phi'$ and until $\phi' \U_I \phi''$}: For eventually $F_I \phi'$  and until $\phi' \U_I \phi''$, we make a similar observation as for always $\G_I \phi'$, i.e., that it corresponds to the application of a countable number of conjunction and disjunction operators. Therefore, it follows that $D^{\F_I \phi'}(x,t)\supseteq \rho^{\F_I \phi'}({x},t)$ and that $D^{\phi' \U_I \phi''}(x,t)\supseteq \rho^{\phi' \U_I \phi''}({x},t)$, which concludes the proof.\qed
\end{proof}

\section{Proof of Theorem \ref{thm:2}}
\label{app:proof_equivalent_representation}
\begin{proof}
As previously mentioned, it is easy to see that the robust semantics $\rho^\mu(x,t)$ are equivalent to the Pareto optimal points of the multi-objective optimization problem with decision variables $\theta:=[\Deltax,\Deltat]$, feasible set $\Theta:=D^\mu(x,t)$, and objective functions $f_1(\theta):=\Deltax$ and $f_2(\theta):=\Deltat$, i.e., that     
\begin{subequations}
\begin{align}
    \rho^\mu(x,t)= \max_{\Deltax,\Deltat}\;&  (\Deltax,\Deltat), \\
\text{s.t.}\; &
  (\Deltax,\Deltat)\in  D^\mu(x,t). \label{eq:MO_predicate_domain}
\end{align}\label{eq:MO_prob___}%
\end{subequations} 
Our proof proceeds by showing an equivalent formulation of the constraint \eqref{eq:MO_predicate_domain}. We first note that the domain $D^\mu(x,t)$ can be re-written as
\begin{align*}
    D^\mu(x,t)\!
    &\!\overset{(a)}{=}\!\!\!\left\lbrace \!(\Deltax\!,\!\Deltat)\!\in\!\overline{\R}_{\geq0}\!\!\times\! \overline{\Z}_{\geq0} \!\!\;\Big\vert\;\!\! (x_{\deltax\!,\deltat},t) \!\models\! \mu,\;
    \!\!\forall\norm{\deltax\!}_\infty\!\!\leq\!\Deltax,\deltat\!\in\![-\Deltat;\!\Deltat]^n \! \right\rbrace\!,\\
    &\!\overset{(b)}{=}\!\!\!\left\lbrace\! (\Deltax\!,\!\Deltat)\!\in\!\overline{\R}_{\geq0}\!\!\times\! \overline{\Z}_{\geq0} \!\!\;\Big\vert\;\!\! (x_{\overline{\deltax}\!,\deltat},t) \!\models\! \mu,\;
    \!\!\forall\|\overline{\deltax}\|\!\!\leq\!\Deltax,\deltat\!\in\![-\Deltat;\!\Deltat]^n   \!\right\rbrace\!,
\end{align*}
where we have written $x_{\overline{\deltax},\deltat}(t)$ with a slight abuse of notation as $x_{\overline{\deltax},\deltat}(t)=[x_1(t-\deltati{1}),\hdots,x_n(t-\deltati{n})]+\overline{\deltax}$. The equality in $(a)$ follows by definition of the domain $D^\mu(x,t)$. The equality in $(b)$ follows for the following reasons. Recall that $(x_{\deltax\!,\deltat},t) \!\models\! \mu$ is defined as $h^\mu(x_{\deltax\!,\deltat}(t))\geq0$. By definition, we also know that $x_{\deltax\!,\deltat}(t) = x_{\deltat}(t) + \deltax(t)$, meaning that $\deltax(t')$ for $t'\neq t$ does not appear in the definition of $D^\mu(x,t)$. We can hence interpret $\overline{\deltax}$ as $\deltax(t)$  and replace $\forall\norm{\deltax\!}_\infty\!\!\leq\!\Deltax$ by $\forall\|\overline{\deltax}\|\!\!\leq\!\Deltax$, from which the equality in $(b)$ follows.

We can now further re-write the domain $D^\mu(x,t)$ as
\begin{equation*}
    D^\mu(x,t) = \left\lbrace (\Deltax,\Deltat)\in\overline{\R}_{\geq0}\times \overline{\Z}_{\geq0} \;\Big\vert\; g^\mu_{x,t}(\Deltax,\Deltat)\geq0  \right\rbrace,
\end{equation*} 
where we introduced the robust predicate margin
\begin{equation*}
    g_{x,t}^\mu(\Deltax,\Deltat) := \inf_{\substack{\norm{\overline{\deltax}}\leq\Deltax\\\deltat\in[-\Deltat;\Deltat]^n}} 
    h^\mu \!\left(\begin{bmatrix}
        x_1(t-\deltati{1})\\\vdots\\x_n(t-\deltati{n})
    \end{bmatrix} + \overline{\deltax}\right).%\label{eq:inner_program}.
\end{equation*}

With this observation in mind, we can re-write the multi-objective optimization problem in equation \eqref{eq:MO_prob___} as the multi-objective optimization problem
\begin{subequations}
\begin{align*}
    \rho^\mu(x,t)= \max_{\Deltax,\Deltat}\;&  (\Deltax,\Deltat),%\label{eq:MO_prob_objfcn} 
    \\
\text{s.t.}\; &
  g_{x,t}^\mu(\Deltax,\Deltat)\ge 0,%\label{eq:MO_prob_constr}
\end{align*}%\label{eq:MO_prob}%
\end{subequations}
which presents an alternative representation for $\rho^\mu(x,t)$ and completes the proof.~\qed
\end{proof}

\section{Algorithm for Computing the Robust Semantics}\label{app:comp_overall}
As outlined previously, once the spatiotemporal robustness of the individual predicates is computed, the spatiotemporal robustness of the overarching specification can be derived through recursively applying min and max operations, encoding the robust semantics in Definition~\ref{def:robust_semantics}.
Algorithm \ref{alg:compute_STR_spec} computes the spatiotemporal robustness envelope
$\Deltax_\phi(\Deltat;t)$ by evaluating the specification $\phi$ bottom--up on its parse tree; see Figure~\ref{fig:spec_parsetree_fighterjet} for an example specification parse tree consisting of three subformulas.

% \paragraph{Algorithm summary}
Line~1 performs a postorder traversal of the specification's parse tree, ensuring that each node $v \in \mathcal{N}$ is processed only after all of its children. This enforces a dynamic-programming structure: every spatiotemporal robustness envelope computed at a node can be reused by its ancestors without recomputation. The remaining lines encode the robust semantics from Definition~\ref{def:robust_semantics}.

Lines~3--5 handle predicate nodes. For each predicate $\mu_v$, the algorithm iterates over the time indices $t'$ that may influence the robustness at the query time $t$, as determined by the node horizon $v.\texttt{horizon}$ (Line~4). At each such $t'$, the predicate-level robustness envelope $\Deltax_v(\cdotx;t')$ is computed once using the appropriate predicate-level algorithm (Line~5; e.g., Algorithm~\ref{alg:compute_STR_pred}). These predicate nodes constitute the leaves of the robustness computation (cf. Fig.~\ref{fig:spec_parsetree_fighterjet}).

Lines~6--9 implement Boolean conjunction. For every relevant time $t'$ (Line~7) and temporal disturbance level $\Deltat$ (Line~8), the robustness envelope of a conjunction is obtained by the pointwise minimum of the child envelopes (Line~9). This mirrors the standard robust semantics $\rho^{\phi_1\wedge\phi_2}=\min\{\rho^{\phi_1},\rho^{\phi_2}\}$, yielding the lower envelope (Definition~\ref{def:min_max}).

Lines~10--13 implement Boolean disjunction analogously. The envelope at the parent node is the pointwise maximum of the child envelopes (Line~13), corresponding to the upper envelope $\rho^{\phi_1\vee\phi_2}=\max\{\rho^{\phi_1},\rho^{\phi_2}\}$.

Lines~14--17 implement the always operator $\G_{[a;b]}$. For each $(t',\Deltat)$, the parent envelope is the minimum over all time-shifted child envelopes on the interval $[a;b]$ (Line~17). This realizes the robust semantics of $\G$ as a temporal minimum, while preserving the monotonicity of the envelope in $\Deltat$.

Lines~18--21 implement the eventually operator $\F_{[a;b]}$. Here, the envelope is obtained by taking the maximum over the time-shifted child envelopes (Line~21), corresponding to the robust semantics of $\F$ as a temporal maximum.

Lines~22--30 implement the until operator $\U_{[a;b]}$ using an efficient single-pass dynamic program. For a fixed $(t',\Deltat)$, the algorithm sweeps the interval $k=a,\ldots,b$ once. Line~28 maintains a running prefix minimum $m=\min_{k\in[a;b]}\Deltax_{c_\ell}(\allowbreak\Deltat;t'+k)$ of the left subformula. Line~29 updates the best achievable value $\best=\max_{k\in[a,b]}\min\{m,\,\Deltax_{c_r}(\Deltat;t'+k)\}$. This directly implements the robust semantics of the until operator (see Definition~\ref{def:robust_semantics}), but avoids the quadratic nested minimization by using prefix minima.

Finally, Line~31 returns the envelope at the root node, yielding $\Deltax_\phi(\cdotx;t)$, the  robustness envelope of the full specification at the query time $t$.

\begin{algorithm}
\caption{Compute specification-level $\mathcal{STR}^\phi(x,t)$ (full version)}
\label{alg:compute_STR_spec}
\begin{algorithmic}[1]
\Require STL-like specification $\phi$, signal $x\colon\Z\to\R^n$, time $t\in\Z$, $\Deltat^{\max}\in\N$
\Ensure Envelope $\mathrm{STR}^\phi_{x,t}=\{(\Deltax_\phi(\Deltat),\Deltat)\}_{\Deltat\in\N_{\geq0}}$

\State $\mathcal{N}\gets \phi.\texttt{nodes\_postorder()}$ \Comment{bottom-up traversal}

\ForAll{$v \in \mathcal{N}$} 
    \color{black!40!white}\If{$v$ is a predicate node with predicate $\mu_v$}
        \For{$t' \in t\oplus v.\texttt{horizon}$} \Comment{only where needed upstream}
            \State $\Deltax_v(\cdotx;t') \gets 
            \texttt{Algorithm\_\ref{alg:compute_STR_pred}}(\mu_v,x,t')$
        \EndFor

    \color{black}\ElsIf{$v$ is Boolean $\wedge$ with children $c_1,c_2$}
        \For{$t' \in t\oplus v.\texttt{horizon}$}
            \For{$\Deltat=0,\ldots,\Deltat^{\max}$}
                \State $\Deltax_v(\Deltat;t') \gets 
                \min\{\Deltax_{c_1}(\Deltat;t'),\,\Deltax_{c_2}(\Deltat;t')\}$
            \EndFor
        \EndFor

    \color{black!40!white}\ElsIf{$v$ is Boolean $\vee$ with children $c_1,c_2$}
        \For{$t' \in t\oplus v.\texttt{horizon}$}
            \For{$\Deltat=0,\ldots,\Deltat^{\max}$}
                \State $\Deltax_v(\Deltat;t') \gets 
                \max\{\Deltax_{c_1}(\Deltat;t'),\,\Deltax_{c_2}(\Deltat;t')\}$
            \EndFor
        \EndFor

    \color{black}\ElsIf{$v$ is $\G_{[a;b]}$ with child $c$}
        \For{$t' \in t\oplus v.\texttt{horizon}$}
            \For{$\Deltat=0,\ldots,\Deltat^{\max}$}
                \State $\Deltax_v(\Deltat;t') \gets 
                \min\limits_{k\in[a;b]}\ \Deltax_{c}(\Deltat;t'+k)$
            \EndFor
        \EndFor

    \color{black!40!white}\ElsIf{$v$ is $\F_{[a;b]}$ with child $c$}
        \For{$t' \in t\oplus v.\texttt{horizon}$}
            \For{$\Deltat=0,\ldots,\Deltat^{\max}$}
                \State $\Deltax_v(\Deltat;t') \gets 
                \max\limits_{k\in[a;b]}\ \Deltax_{c}(\Deltat;t'+k)$
            \EndFor
        \EndFor

    \color{black}\ElsIf{$v$ is $\ \U_{[a;b]}$ with children $c_\ell$ (left), $c_r$ (right)}
        \For{$t' \in t\oplus v.\texttt{horizon}$}
            \For{$\Deltat=0,\ldots,\Deltat^{\max}$} \Comment{efficient prefix-minimum DP for each $\Deltat$}
                \State $m \gets +\infty$ \Comment{running prefix minimum of left child}
                \State $\best \gets -\infty$
                \For{$k=0,\ldots,b$}
                    \State $m \gets \min\{m,\,\Deltax_{c_\ell}(\Deltat;t'+k)\}$
                    \If{$k \ge a$}
                      \State $\best \gets \max\{\best,\,\min\{m,\,\Deltax_{c_r}(\Deltat;t'+k)\}\}$
                    \EndIf
                \EndFor
                \State $\Deltax_v(\Deltat;t') \gets \best$
            \EndFor
        \EndFor
    \EndIf
\color{black}\EndFor

\State \Return $\Deltax_{\phi}(\cdotx;t)$
\end{algorithmic}
\end{algorithm}

\end{document}

%% file: Figures/graph.pdf_tex
%% Creator: Inkscape 1.3.2 (091e20e, 2023-11-25), www.inkscape.org
%% Accompanies image file 'graph.pdf' (pdf, eps, ps)
\begingroup%
  \makeatletter%
  \providecommand\color[2][]{%
    \errmessage{(Inkscape) Color is used for the text in Inkscape, but the package 'color.sty' is not loaded}%
    \renewcommand\color[2][]{}%
  }%
  \providecommand\transparent[1]{%
    \errmessage{(Inkscape) Transparency is used (non-zero) for the text in Inkscape, but the package 'transparent.sty' is not loaded}%
    \renewcommand\transparent[1]{}%
  }%
  \providecommand\rotatebox[2]{#2}%
  \newcommand*\fsize{\dimexpr\f@size pt\relax}%
  \newcommand*\lineheight[1]{\fontsize{\fsize}{#1\fsize}\selectfont}%
  \ifx\svgwidth\undefined%
    \setlength{\unitlength}{313bp}%
    \ifx\svgscale\undefined%
      \relax%
    \else%
      \setlength{\unitlength}{\unitlength * \real{\svgscale}}%
    \fi%
  \else%
    \setlength{\unitlength}{\svgwidth}%
  \fi%
  \global\let\svgwidth\undefined%
  \global\let\svgscale\undefined%
  \makeatother%
  \begin{picture}(1,0.61980831)%
    \lineheight{1}%
    \setlength\tabcolsep{0pt}%
    \put(0,0){\includegraphics[width=\unitlength,page=1]{graph.pdf}}%
    \put(0.0798722,0.61501597){\color[rgb]{0,0,0}\makebox(0,0)[t]{\smash{\begin{tabular}[t]{c}$\Delta_X$\end{tabular}}}}%
    \put(0.02164506,0.30397304){\color[rgb]{0,0,0}\makebox(0,0)[t]{\smash{\begin{tabular}[t]{c}$\color{myorange}\Delta_X''$\end{tabular}}}}%
    \put(0.02673691,0.14570145){\color[rgb]{0,0,0}\makebox(0,0)[t]{\smash{\begin{tabular}[t]{c}$\color{myred}\Delta_X'$\end{tabular}}}}%
    \put(0.16791102,0.0169998){\color[rgb]{0,0,0}\makebox(0,0)[t]{\smash{\begin{tabular}[t]{c}$\color{myorange}\Delta_T''$\end{tabular}}}}%
    \put(0.349779,0.0169998){\color[rgb]{0,0,0}\makebox(0,0)[t]{\smash{\begin{tabular}[t]{c}$\color{myred}\Delta_T'$\end{tabular}}}}%
    \put(0.7498773,0.06257176){\color[rgb]{0,0,0}\makebox(0,0)[t]{\smash{\begin{tabular}[t]{c}$\Delta_T$\end{tabular}}}}%
    \put(0.89120259,0.21633677){\color[rgb]{0,0,0}\makebox(0,0)[t]{\smash{\begin{tabular}[t]{c}Temporal robustness\end{tabular}}}}%
    \put(0.82776893,0.4664537){\color[rgb]{0,0,0}\makebox(0,0)[t]{\smash{\begin{tabular}[t]{c}Spatiotemporal\\robustness\end{tabular}}}}%
    \put(0.49102496,0.58669257){\color[rgb]{0,0,0}\makebox(0,0)[t]{\smash{\begin{tabular}[t]{c}Spatial robustness\end{tabular}}}}%
  \end{picture}%
\endgroup%

%% file: Figures/downward_closure.pdf_tex
%% Creator: Inkscape 1.3.2 (091e20e, 2023-11-25), www.inkscape.org
%% Accompanies image file 'downward_closure.pdf' (pdf, eps, ps)
\begingroup%
  \makeatletter%
  \providecommand\color[2][]{%
    \errmessage{(Inkscape) Color is used for the text in Inkscape, but the package 'color.sty' is not loaded}%
    \renewcommand\color[2][]{}%
  }%
  \providecommand\transparent[1]{%
    \errmessage{(Inkscape) Transparency is used (non-zero) for the text in Inkscape, but the package 'transparent.sty' is not loaded}%
    \renewcommand\transparent[1]{}%
  }%
  \providecommand\rotatebox[2]{#2}%
  \newcommand*\fsize{\dimexpr\f@size pt\relax}%
  \newcommand*\lineheight[1]{\fontsize{\fsize}{#1\fsize}\selectfont}%
  \ifx\svgwidth\undefined%
    \setlength{\unitlength}{410.25bp}%
    \ifx\svgscale\undefined%
      \relax%
    \else%
      \setlength{\unitlength}{\unitlength * \real{\svgscale}}%
    \fi%
  \else%
    \setlength{\unitlength}{\svgwidth}%
  \fi%
  \global\let\svgwidth\undefined%
  \global\let\svgscale\undefined%
  \makeatother%
  \begin{picture}(1,0.31444241)%
    \lineheight{1}%
    \setlength\tabcolsep{0pt}%
    \put(0,0){\includegraphics[width=\unitlength,page=1]{downward_closure.pdf}}%
    \put(0.06059936,0.17772033){\color[rgb]{0,0,0}\makebox(0,0)[t]{\smash{\begin{tabular}[t]{c}$\color{myred}D^1$\end{tabular}}}}%
    \put(0.00794818,0.30645852){\color[rgb]{0,0,0}\makebox(0,0)[t]{\smash{\begin{tabular}[t]{c}$\Delta_X$\end{tabular}}}}%
    \put(0.04136122,0.09005679){\color[rgb]{0,0,0}\makebox(0,0)[t]{\smash{\begin{tabular}[t]{c}$\color{myred}D^1_\downarrow$\end{tabular}}}}%
    \put(0.21057508,0.04997638){\color[rgb]{0,0,0}\makebox(0,0)[t]{\smash{\begin{tabular}[t]{c}$\color{myblue}D^2$\end{tabular}}}}%
    \put(0.11499998,0.02944782){\color[rgb]{0,0,0}\makebox(0,0)[t]{\smash{\begin{tabular}[t]{c}$\color{myblue}D^2_\downarrow$\end{tabular}}}}%
    \put(0.31112666,0.0022995){\color[rgb]{0,0,0}\makebox(0,0)[t]{\smash{\begin{tabular}[t]{c}$\Delta_T$\end{tabular}}}}%
    \put(0.54455856,0.17772033){\color[rgb]{0,0,0}\makebox(0,0)[t]{\smash{\begin{tabular}[t]{c}$\max\{D^1_\downarrow \cap D^2_\downarrow\}$\end{tabular}}}}%
    \put(0.88209291,0.17772033){\color[rgb]{0,0,0}\makebox(0,0)[t]{\smash{\begin{tabular}[t]{c}$\max\{D^1 \cup D^2\}$\end{tabular}}}}%
    \put(0.35055325,0.30645852){\color[rgb]{0,0,0}\makebox(0,0)[t]{\smash{\begin{tabular}[t]{c}$\Delta_X$\end{tabular}}}}%
    \put(0.69361401,0.30645852){\color[rgb]{0,0,0}\makebox(0,0)[t]{\smash{\begin{tabular}[t]{c}$\Delta_X$\end{tabular}}}}%
    \put(0.65329872,0.0022995){\color[rgb]{0,0,0}\makebox(0,0)[t]{\smash{\begin{tabular}[t]{c}$\Delta_T$\end{tabular}}}}%
    \put(0.99634369,0.00365631){\color[rgb]{0,0,0}\makebox(0,0)[t]{\smash{\begin{tabular}[t]{c}$\Delta_T$\end{tabular}}}}%
  \end{picture}%
\endgroup%